\documentclass[journal]{IEEEtran}

\usepackage{cite}
\usepackage{amsmath,amssymb,amsfonts}
\usepackage{amsthm}
\newtheorem{proposition}{Proposition}
\usepackage{dsfont}
\usepackage{tikz}
\usetikzlibrary{arrows.meta,positioning,fit,calc,patterns}
\usepackage{graphicx}
\usepackage{algorithm}
\usepackage{algpseudocode}
\usepackage{booktabs}
\usepackage{multirow}
\usepackage{threeparttable}
\usepackage{array}
\usepackage{placeins}
\usepackage{xcolor}
\usepackage{url}
\usepackage{balance}
\usepackage{pdfpages}

\newcommand{\MotFifoFree}{113}
\newcommand{\MotFifoLag}{206}
\newcommand{\MotFifoInfl}{82}
\newcommand{\MotOptFree}{100}
\newcommand{\MotOptLag}{198}
\newcommand{\MotOptInfl}{98}
\newcommand{\MotRepaired}{243}
\newcommand{\MotRepairPenalty}{18}
\newcommand{\nTest}{100}
\newcommand{\TestLagInflation}{67.2}
\newcommand{\NumSeeds}{three}
\newcommand{\MtightSptMean}{267.3}
\newcommand{\nOpsPerModule}{22}
\newcommand{\nStationTypes}{9}
\newcommand{\nOpTypes}{5}
\newcommand{\nRoutingClasses}{4}
\newcommand{\FactoryDefault}{25}
\newcommand{\FactoryTight}{15}
\newcommand{\CureLo}{24}\newcommand{\CureHi}{48}
\newcommand{\PondLo}{24}\newcommand{\PondHi}{48}
\newcommand{\DryLo}{12}\newcommand{\DryHi}{24}
\newcommand{\RegressMean}{408.40}
\newcommand{\HgnnRegress}{111.67}
\newcommand{\PlanRegress}{407.86}
\newcommand{\AnchorModules}{3,034}
\newcommand{\AnchorTowers}{two}
\newcommand{\AnchorFloors}{56}
\newcommand{\ModulesPerFloor}{27}
\newcommand{\MotModules}{five}
\newcommand{\BatchSmall}{10}
\newcommand{\BatchProdLo}{50}
\newcommand{\BatchProdHi}{80}
\newcommand{\nRules}{eight}
\newcommand{\nLagRules}{four}
\newcommand{\MaBudget}{300}
\newcommand{\StochCvHi}{0.3}
\newcommand{\nDisruptTypes}{six}
\newcommand{\AnytimeBudget}{600}
\newcommand{\AnytimeCores}{four}
\newcommand{\SlabLatencyEighty}{\LatSlabEighty}
\newcommand{\AnytimeBudgetMid}{300}
\newcommand{\FtLagScale}{1.5}
\newcommand{\RepairDanielTen}{308.1}
\newcommand{\DispatchDanielTen}{213.7}
\newcommand{\RepairPenDanielTen}{44.2}
\newcommand{\RepairDanielTwenty}{449.7}
\newcommand{\DispatchDanielTwenty}{259.7}
\newcommand{\RepairPenDanielTwenty}{73.2}
\newcommand{\RepairDanielThirty}{593.3}
\newcommand{\DispatchDanielThirty}{311.8}
\newcommand{\RepairPenDanielThirty}{90.3}
\newcommand{\RepairDanielFifty}{883.1}
\newcommand{\DispatchDanielFifty}{415.0}
\newcommand{\RepairPenDanielFifty}{112.8}
\newcommand{\RepairDanielEighty}{1327.7}
\newcommand{\DispatchDanielEighty}{571.0}
\newcommand{\RepairPenDanielEighty}{132.5}
\newcommand{\RepairPlanTen}{304.5}
\newcommand{\DispatchPlanTen}{215.1}
\newcommand{\RepairPenPlanTen}{41.6}
\newcommand{\RepairPlanTwenty}{444.2}
\newcommand{\DispatchPlanTwenty}{262.0}
\newcommand{\RepairPenPlanTwenty}{69.5}
\newcommand{\RepairPlanThirty}{607.8}
\newcommand{\DispatchPlanThirty}{316.9}
\newcommand{\RepairPenPlanThirty}{91.8}
\newcommand{\RepairPlanFifty}{859.6}
\newcommand{\DispatchPlanFifty}{415.9}
\newcommand{\RepairPenPlanFifty}{106.7}
\newcommand{\RepairPlanEighty}{1307.5}
\newcommand{\DispatchPlanEighty}{587.7}
\newcommand{\RepairPenPlanEighty}{122.5}
\newcommand{\RepairHgnnTen}{316.0}
\newcommand{\DispatchHgnnTen}{214.2}
\newcommand{\RepairPenHgnnTen}{47.5}
\newcommand{\RepairHgnnTwenty}{478.9}
\newcommand{\DispatchHgnnTwenty}{279.4}
\newcommand{\RepairPenHgnnTwenty}{71.4}
\newcommand{\RepairHgnnThirty}{636.2}
\newcommand{\DispatchHgnnThirty}{348.1}
\newcommand{\RepairPenHgnnThirty}{82.8}
\newcommand{\RepairHgnnFifty}{927.6}
\newcommand{\DispatchHgnnFifty}{492.6}
\newcommand{\RepairPenHgnnFifty}{88.3}
\newcommand{\RepairHgnnEighty}{1329.2}
\newcommand{\DispatchHgnnEighty}{681.5}
\newcommand{\RepairPenHgnnEighty}{95.0}
\newcommand{\RepairPenLo}{42}
\newcommand{\RepairPenHi}{133}
\newcommand{\SptTen}{216.0}
\newcommand{\LlfTen}{215.7}
\newcommand{\MaTen}{209.5}
\newcommand{\CpColdTen}{203.2}
\newcommand{\CpWarmSptTen}{203.2}
\newcommand{\HstarTen}{211.3}
\newcommand{\HstarStdTen}{0.4}
\newcommand{\HstarVsSptGainTen}{2.2}
\newcommand{\HstarVsCpWarmSptTen}{4.0}
\newcommand{\nInstTen}{100}
\newcommand{\SptTwenty}{265.1}
\newcommand{\LlfTwenty}{265.0}
\newcommand{\MaTwenty}{260.5}
\newcommand{\CpColdTwenty}{253.7}
\newcommand{\CpWarmSptTwenty}{253.4}
\newcommand{\HstarTwenty}{255.0}
\newcommand{\HstarStdTwenty}{0.8}
\newcommand{\HstarVsCpWarmSptTwenty}{0.7}
\newcommand{\nInstTwenty}{50}
\newcommand{\SptThirty}{324.9}
\newcommand{\LlfThirty}{326.0}
\newcommand{\MaThirty}{322.1}
\newcommand{\CpColdThirty}{318.9}
\newcommand{\CpWarmSptThirty}{318.5}
\newcommand{\HstarThirty}{308.7}
\newcommand{\HstarStdThirty}{1.9}
\newcommand{\nInstThirty}{20}
\newcommand{\SptFifty}{449.0}
\newcommand{\LlfFifty}{450.8}
\newcommand{\MaFifty}{444.6}
\newcommand{\CpColdFifty}{464.2}
\newcommand{\CpWarmSptFifty}{444.1}
\newcommand{\CpWarmPolFifty}{409.4}
\newcommand{\CpWarmPolHintFifty}{412.5}
\newcommand{\CpWarmPolGainFifty}{0.7}
\newcommand{\HstarFifty}{414.9}
\newcommand{\HstarStdFifty}{1.9}
\newcommand{\HstarVsMaGainFifty}{6.7}
\newcommand{\HstarVsCpWarmSptGainFifty}{6.6}
\newcommand{\nInstFifty}{50}
\newcommand{\SptEighty}{620.3}
\newcommand{\LlfEighty}{623.6}
\newcommand{\MaEighty}{615.8}
\newcommand{\CpColdEighty}{724.5}
\newcommand{\CpWarmSptEighty}{617.1}
\newcommand{\CpWarmPolEighty}{563.9}
\newcommand{\CpWarmPolHintEighty}{566.8}
\newcommand{\CpWarmPolGainEighty}{0.5}
\newcommand{\HstarEighty}{570.3}
\newcommand{\HstarStdEighty}{3.7}
\newcommand{\HstarVsSptGainEighty}{8.1}
\newcommand{\HstarVsMaGainEighty}{7.4}
\newcommand{\HstarVsCpWarmSptGainEighty}{7.6}
\newcommand{\nInstEighty}{50}
\newcommand{\CpOptTen}{83}
\newcommand{\CpLongWarmFifty}{439.7}
\newcommand{\CpLongWarmEighty}{616.0}
\newcommand{\SptIndTen}{218.6}
\newcommand{\HstarIndTen}{214.4}
\newcommand{\HstarIndVsSptGainTen}{1.9}
\newcommand{\SptIndTwenty}{265.0}
\newcommand{\HstarIndTwenty}{255.4}
\newcommand{\SptIndFifty}{448.8}
\newcommand{\HstarIndFifty}{415.1}
\newcommand{\SptIndEighty}{623.5}
\newcommand{\HstarIndEighty}{572.2}
\newcommand{\HstarIndVsSptGainEighty}{8.2}
\newcommand{\FlowCpLineFifty}{436.2}
\newcommand{\FlowCpLineEighty}{634.3}
\newcommand{\LagRatioMean}{0.85}
\newcommand{\TrainLagFreeTen}{213.7}
\newcommand{\TrainLagFreeStdTen}{0.6}
\newcommand{\TrainLagFreeTwenty}{259.7}
\newcommand{\TrainLagFreeStdTwenty}{1.9}
\newcommand{\TrainLagFreeThirty}{311.8}
\newcommand{\TrainLagFreeStdThirty}{1.9}
\newcommand{\TrainLagFreeFifty}{415.0}
\newcommand{\TrainLagFreeStdFifty}{2.5}
\newcommand{\TrainLagFreeEighty}{571.0}
\newcommand{\TrainLagFreeStdEighty}{2.6}
\newcommand{\TrainLagFreeTight}{262.7}
\newcommand{\TrainLagFreeStdTight}{0.5}
\newcommand{\TrainHalfTen}{211.3}
\newcommand{\TrainHalfStdTen}{0.4}
\newcommand{\TrainHalfTwenty}{255.0}
\newcommand{\TrainHalfStdTwenty}{0.8}
\newcommand{\TrainHalfThirty}{308.7}
\newcommand{\TrainHalfStdThirty}{1.9}
\newcommand{\TrainHalfFifty}{414.9}
\newcommand{\TrainHalfStdFifty}{1.9}
\newcommand{\TrainHalfEighty}{570.3}
\newcommand{\TrainHalfStdEighty}{3.7}
\newcommand{\TrainHalfTight}{257.4}
\newcommand{\TrainHalfStdTight}{1.3}
\newcommand{\TrainFullTen}{211.6}
\newcommand{\TrainFullStdTen}{0.4}
\newcommand{\TrainFullTwenty}{258.5}
\newcommand{\TrainFullStdTwenty}{0.7}
\newcommand{\TrainFullThirty}{317.0}
\newcommand{\TrainFullStdThirty}{2.5}
\newcommand{\TrainFullFifty}{428.9}
\newcommand{\TrainFullStdFifty}{7.5}
\newcommand{\TrainFullEighty}{592.2}
\newcommand{\TrainFullStdEighty}{11.9}
\newcommand{\TrainFullTight}{258.6}
\newcommand{\TrainFullStdTight}{0.7}
\newcommand{\TrainFullAidsTen}{211.2}
\newcommand{\TrainFullAidsStdTen}{0.2}
\newcommand{\TrainFullAidsTwenty}{260.0}
\newcommand{\TrainFullAidsStdTwenty}{2.2}
\newcommand{\TrainFullAidsThirty}{319.0}
\newcommand{\TrainFullAidsStdThirty}{3.6}
\newcommand{\TrainFullAidsFifty}{435.6}
\newcommand{\TrainFullAidsStdFifty}{5.5}
\newcommand{\TrainFullAidsEighty}{608.3}
\newcommand{\TrainFullAidsStdEighty}{5.7}
\newcommand{\TrainFullAidsTight}{257.0}
\newcommand{\TrainFullAidsStdTight}{1.5}
\newcommand{\TrainTightTen}{211.1}
\newcommand{\TrainTightStdTen}{0.1}
\newcommand{\TrainTightTwenty}{255.0}
\newcommand{\TrainTightStdTwenty}{0.8}
\newcommand{\TrainTightThirty}{309.4}
\newcommand{\TrainTightStdThirty}{1.4}
\newcommand{\TrainTightFifty}{420.0}
\newcommand{\TrainTightStdFifty}{1.2}
\newcommand{\TrainTightEighty}{583.8}
\newcommand{\TrainTightStdEighty}{4.5}
\newcommand{\TrainTightTight}{253.1}
\newcommand{\TrainTightStdTight}{0.1}
\newcommand{\PlanHalfTen}{211.4}
\newcommand{\PlanHalfStdTen}{0.2}
\newcommand{\PlanHalfTwenty}{256.1}
\newcommand{\PlanHalfStdTwenty}{1.1}
\newcommand{\PlanHalfThirty}{309.8}
\newcommand{\PlanHalfStdThirty}{3.6}
\newcommand{\PlanHalfFifty}{420.7}
\newcommand{\PlanHalfStdFifty}{4.2}
\newcommand{\PlanHalfEighty}{578.7}
\newcommand{\PlanHalfStdEighty}{7.9}
\newcommand{\PlanHalfTight}{258.7}
\newcommand{\PlanHalfStdTight}{0.8}
\newcommand{\PlanLagFreeTen}{215.1}
\newcommand{\PlanLagFreeTwenty}{262.0}
\newcommand{\PlanLagFreeThirty}{316.9}
\newcommand{\PlanLagFreeFifty}{415.9}
\newcommand{\PlanLagFreeEighty}{587.7}
\newcommand{\PlanLagFreeTight}{268.5}
\newcommand{\PlanFullTen}{211.4}
\newcommand{\PlanFullTwenty}{257.6}
\newcommand{\PlanFullThirty}{315.3}
\newcommand{\PlanFullFifty}{431.4}
\newcommand{\PlanFullEighty}{621.5}
\newcommand{\PlanFullTight}{258.0}
\newcommand{\PlanFullAidsTen}{211.4}
\newcommand{\PlanFullAidsTwenty}{260.8}
\newcommand{\PlanFullAidsThirty}{322.7}
\newcommand{\PlanFullAidsFifty}{441.8}
\newcommand{\PlanFullAidsEighty}{613.1}
\newcommand{\PlanFullAidsTight}{262.6}
\newcommand{\TrainQuarterEighty}{562.9}
\newcommand{\TrainThreeQuarterEighty}{583.2}
\newcommand{\ContTenLagZero}{1.00}
\newcommand{\ContTenLagQuarter}{0.50}
\newcommand{\ContTenLagHalf}{0.32}
\newcommand{\ContTenLagThreeQuarter}{0.21}
\newcommand{\ContTenLagFull}{0.15}
\newcommand{\ContTwenty}{0.31}
\newcommand{\ContThirty}{0.48}
\newcommand{\ContFifty}{0.63}
\newcommand{\ContEighty}{0.73}
\newcommand{\ContTight}{0.36}
\newcommand{\ContTenLagFullPct}{15}
\newcommand{\ContFiftyPct}{63}
\newcommand{\ContEightyPct}{73}
\newcommand{\BoundDiffTen}{-0.1}
\newcommand{\BoundDiffTwenty}{-2.5}
\newcommand{\BoundCiLoTwenty}{-3.7}
\newcommand{\BoundCiHiTwenty}{-1.2}
\newcommand{\BoundDiffThirty}{-3.8}
\newcommand{\BoundCiLoThirty}{-6.6}
\newcommand{\BoundCiHiThirty}{-1.4}
\newcommand{\BoundDiffFifty}{-0.5}
\newcommand{\BoundDiffEighty}{-1.7}
\newcommand{\StochReactTenCvOne}{212.2}
\newcommand{\StochReactTenCvTwo}{215.2}
\newcommand{\StochReactTenCvThree}{218.8}
\newcommand{\StochRepairTenCvOne}{217.8}
\newcommand{\StochRepairTenCvTwo}{227.4}
\newcommand{\StochRepairTenCvThree}{236.1}
\newcommand{\StochPerfInfoTenCvOne}{212.3}
\newcommand{\StochPerfInfoTenCvTwo}{215.2}
\newcommand{\StochPerfInfoTenCvThree}{218.8}
\newcommand{\StochSptTenCvOne}{217.6}
\newcommand{\StochSptTenCvTwo}{220.2}
\newcommand{\StochSptTenCvThree}{224.1}
\newcommand{\StochCpRepairTenCvOne}{211.8}
\newcommand{\StochCpRepairTenCvTwo}{222.5}
\newcommand{\StochCpRepairTenCvThree}{231.4}
\newcommand{\StochCpRollTenCvOne}{206.8}
\newcommand{\StochCpRollTenCvTwo}{211.7}
\newcommand{\StochCpRollTenCvThree}{217.1}
\newcommand{\StochReactFiftyCvOne}{410.5}
\newcommand{\StochReactFiftyCvTwo}{411.1}
\newcommand{\StochReactFiftyCvThree}{412.4}
\newcommand{\StochRepairFiftyCvOne}{424.9}
\newcommand{\StochRepairFiftyCvTwo}{448.7}
\newcommand{\StochRepairFiftyCvThree}{471.4}
\newcommand{\StochPerfInfoFiftyCvOne}{410.8}
\newcommand{\StochPerfInfoFiftyCvTwo}{410.9}
\newcommand{\StochPerfInfoFiftyCvThree}{411.1}
\newcommand{\StochSptFiftyCvOne}{448.9}
\newcommand{\StochSptFiftyCvTwo}{448.4}
\newcommand{\StochSptFiftyCvThree}{448.1}
\newcommand{\StochCpRepairFiftyCvOne}{455.8}
\newcommand{\StochCpRepairFiftyCvTwo}{479.2}
\newcommand{\StochCpRepairFiftyCvThree}{504.5}
\newcommand{\StochGapPiMax}{0.32}
\newcommand{\StochCpRollReplanS}{2.8}
\newcommand{\DisLagSlabTen}{221.0}
\newcommand{\DisLagGaTen}{220.9}
\newcommand{\DisLagRsTen}{232.6}
\newcommand{\DisBrkSlabTen}{210.9}
\newcommand{\DisBrkGaTen}{210.0}
\newcommand{\DisBrkRsTen}{210.9}
\newcommand{\DisUrgSlabTen}{230.7}
\newcommand{\DisUrgGaTen}{231.0}
\newcommand{\DisUrgRsTen}{327.9}
\newcommand{\DisMatSlabTen}{231.0}
\newcommand{\DisMatGaTen}{232.2}
\newcommand{\DisMatRsTen}{243.5}
\newcommand{\DisPtdSlabTen}{211.1}
\newcommand{\DisPtdGaTen}{210.1}
\newcommand{\DisPtdRsTen}{211.6}
\newcommand{\DisCanSlabTen}{206.5}
\newcommand{\DisCanGaTen}{206.0}
\newcommand{\DisCanRsTen}{207.8}
\newcommand{\DisNervSlabTen}{50.0}
\newcommand{\DisTimeSlabTen}{0.47}
\newcommand{\DisNervGaTen}{57.6}
\newcommand{\DisTimeGaTen}{5.0}
\newcommand{\DisNervRsTen}{25.9}
\newcommand{\DisTimeRsTen}{0.00}
\newcommand{\DisLagSlabFifty}{411.6}
\newcommand{\DisLagGaFifty}{415.9}
\newcommand{\DisLagRsFifty}{424.6}
\newcommand{\DisBrkSlabFifty}{412.3}
\newcommand{\DisBrkGaFifty}{411.4}
\newcommand{\DisBrkRsFifty}{412.0}
\newcommand{\DisUrgSlabFifty}{415.9}
\newcommand{\DisUrgGaFifty}{443.8}
\newcommand{\DisUrgRsFifty}{544.3}
\newcommand{\DisMatSlabFifty}{409.7}
\newcommand{\DisMatGaFifty}{428.1}
\newcommand{\DisMatRsFifty}{442.3}
\newcommand{\DisPtdSlabFifty}{412.1}
\newcommand{\DisPtdGaFifty}{413.9}
\newcommand{\DisPtdRsFifty}{415.9}
\newcommand{\DisCanSlabFifty}{405.7}
\newcommand{\DisCanGaFifty}{408.8}
\newcommand{\DisCanRsFifty}{411.1}
\newcommand{\DisNervSlabFifty}{563.2}
\newcommand{\DisTimeSlabFifty}{4.1}
\newcommand{\DisNervGaFifty}{285.2}
\newcommand{\DisTimeGaFifty}{5.0}
\newcommand{\DisNervRsFifty}{203.3}
\newcommand{\DisTimeRsFifty}{0.00}
\newcommand{\BufHstarTenFree}{210.9}
\newcommand{\BufSptTenFree}{216.0}
\newcommand{\BufCpTenFree}{203.4}
\newcommand{\BufHstarTenSix}{215.5}
\newcommand{\BufSptTenSix}{221.6}
\newcommand{\BufCpTenSix}{206.1}
\newcommand{\BufHstarTenFour}{270.6}
\newcommand{\BufSptTenFour}{274.5}
\newcommand{\BufCpTenFour}{245.4}
\newcommand{\BufHstarTenThree}{354.1}
\newcommand{\BufSptTenThree}{345.0}
\newcommand{\BufCpTenThree}{302.5}
\newcommand{\BufPeakFreeTen}{9}
\newcommand{\BufHstarFiftyFree}{411.5}
\newcommand{\BufSptFiftyFree}{448.6}
\newcommand{\BufCpFiftyFree}{444.2}
\newcommand{\BufHstarFiftySixteen}{411.3}
\newcommand{\BufSptFiftySixteen}{447.5}
\newcommand{\BufCpFiftySixteen}{443.7}
\newcommand{\BufHstarFiftyTwelve}{480.7}
\newcommand{\BufSptFiftyTwelve}{480.2}
\newcommand{\BufCpFiftyTwelve}{466.5}
\newcommand{\BufHstarFiftyEight}{649.0}
\newcommand{\BufSptFiftyEight}{663.3}
\newcommand{\BufCpFiftyEight}{593.0}
\newcommand{\BufHstarFiftyFour}{1115.9}
\newcommand{\BufSptFiftyFour}{1133.7}
\newcommand{\BufCpFiftyFour}{1083.5}
\newcommand{\BufPeakFreeFifty}{21}
\newcommand{\VarSptTen}{208.1}
\newcommand{\VarHstarZsTen}{205.5}
\newcommand{\VarTrainedTen}{202.9}
\newcommand{\VarSptTwenty}{252.8}
\newcommand{\VarHstarZsTwenty}{244.2}
\newcommand{\VarTrainedTwenty}{246.0}
\newcommand{\VarSptFifty}{428.1}
\newcommand{\VarHstarZsFifty}{394.9}
\newcommand{\VarTrainedFifty}{418.6}
\newcommand{\VarSptEighty}{594.5}
\newcommand{\VarHstarZsEighty}{543.1}
\newcommand{\VarTrainedEighty}{598.0}
\newcommand{\FtTightZs}{258.4}
\newcommand{\FtTightAfter}{254.2}
\newcommand{\FtTightScratch}{253.1}
\newcommand{\FtLagZsTen}{265.9}
\newcommand{\FtLagAfterTen}{266.2}
\newcommand{\FtLagSptTen}{269.8}
\newcommand{\FtLagZsFifty}{447.4}
\newcommand{\FtLagAfterFifty}{447.4}
\newcommand{\FtLagSptFifty}{485.1}
\newcommand{\FtUpdates}{400}
\newcommand{\RbisSptTen}{-0.71}
\newcommand{\PvalSptHolmThirty}{$8.8\times10^{-5}$}
\newcommand{\RbisSptEighty}{-1.00}
\newcommand{\PvalMaHolmTen}{$1.6\times10^{-3}$}
\newcommand{\PvalMaHolmFifty}{$3.8\times10^{-9}$}
\newcommand{\RbisCpWarmTen}{1.00}
\newcommand{\PvalCpWarmHolmTen}{$1.9\times10^{-17}$}
\newcommand{\RbisCpWarmFifty}{-1.00}
\newcommand{\PvalCpWarmHolmFifty}{$3.0\times10^{-9}$}
\newcommand{\FullVsHalfFifty}{3.4}
\newcommand{\FullVsHalfEighty}{3.8}
\newcommand{\LongSlabFifty}{414.5}
\newcommand{\LongSlabEighty}{568.6}
\newcommand{\StabTimeMaxTen}{2.7}
\newcommand{\FlowGainSptTen}{9.3}
\newcommand{\FlowGainCpTen}{9.0}
\newcommand{\DisLagResFiveTen}{218.0}
\newcommand{\DisLagResThirtyTen}{218.0}
\newcommand{\DisBrkResFiveTen}{207.6}
\newcommand{\DisBrkResThirtyTen}{207.6}
\newcommand{\DisUrgResFiveTen}{228.4}
\newcommand{\DisUrgResThirtyTen}{228.4}
\newcommand{\DisMatResFiveTen}{230.4}
\newcommand{\DisMatResThirtyTen}{230.4}
\newcommand{\DisPtdResFiveTen}{207.7}
\newcommand{\DisPtdResThirtyTen}{207.7}
\newcommand{\DisCanResFiveTen}{204.2}
\newcommand{\DisCanResThirtyTen}{204.2}
\newcommand{\DisLagResFiveFifty}{421.6}
\newcommand{\DisLagResThirtyFifty}{414.5}
\newcommand{\DisBrkResFiveFifty}{411.4}
\newcommand{\DisBrkResThirtyFifty}{410.8}
\newcommand{\DisUrgResFiveFifty}{439.7}
\newcommand{\DisUrgResThirtyFifty}{420.3}
\newcommand{\DisMatResFiveFifty}{438.3}
\newcommand{\DisMatResThirtyFifty}{422.6}
\newcommand{\DisPtdResFiveFifty}{415.1}
\newcommand{\DisPtdResThirtyFifty}{413.4}
\newcommand{\DisCanResFiveFifty}{409.1}
\newcommand{\DisCanResThirtyFifty}{407.5}
\newcommand{\DisNervResFiveTen}{68.5}
\newcommand{\DisTimeResFiveTen}{0.28}
\newcommand{\DisNervResThirtyTen}{68.0}
\newcommand{\DisTimeResThirtyTen}{0.39}
\newcommand{\DisNervResFiveFifty}{198.4}
\newcommand{\DisTimeResFiveFifty}{5.0}
\newcommand{\DisNervResThirtyFifty}{253.8}
\newcommand{\DisTimeResThirtyFifty}{30.0}
\newcommand{\ResGapMaxFifty}{3.1}
\newcommand{\ResFiveGapMaxFifty}{7.0}
\newcommand{\ResOptMinTen}{97}
\newcommand{\ResGainMaxTen}{3.4}
\newcommand{\ResGainMinTen}{0.6}
\newcommand{\LatSlabTen}{0.7}
\newcommand{\LatSlabFifty}{6.2}
\newcommand{\LatSlabEighty}{13.0}
\newcommand{\LatSlabMemMb}{515}
\newcommand{\TrainGpuHours}{2.8}
\newcommand{\FtGpuMinutes}{34}
\newcommand{\CpGapMinFifty}{11}
\newcommand{\BoundTightAvgAugTen}{0.98}
\newcommand{\BoundTightAvgFreeTen}{0.65}
\newcommand{\BoundTightAvgAugEighty}{0.70}
\newcommand{\BoundTightAvgFreeEighty}{0.56}
\newcommand{\ShellGainSteelTen}{4.4}
\newcommand{\ShellGainRcFifty}{2.0}
\newcommand{\CpReachMedColdTen}{1.2}
\newcommand{\DisLagLlfTen}{221.2}
\newcommand{\DisBrkLlfTen}{211.1}
\newcommand{\DisUrgLlfTen}{230.7}
\newcommand{\DisMatLlfTen}{231.0}
\newcommand{\DisPtdLlfTen}{211.4}
\newcommand{\DisCanLlfTen}{206.7}
\newcommand{\DisNervLlfTen}{103.7}
\newcommand{\DisTimeLlfTen}{0.09}
\newcommand{\DisLagLlfFifty}{421.9}
\newcommand{\DisBrkLlfFifty}{423.1}
\newcommand{\DisUrgLlfFifty}{428.1}
\newcommand{\DisMatLlfFifty}{424.5}
\newcommand{\DisPtdLlfFifty}{425.4}
\newcommand{\DisCanLlfFifty}{418.3}
\newcommand{\DisNervLlfFifty}{918.6}
\newcommand{\DisTimeLlfFifty}{0.79}
\newcommand{\RuleGapMinFifty}{9.9}
\newcommand{\RuleGapMaxFifty}{11.7}
\newcommand{\SevMarginEightFifty}{2.2}
\newcommand{\SevMarginTwentyFourFifty}{8.4}
\newcommand{\StabInflLagFifty}{1.1}
\newcommand{\StabNervLagFifty}{29.5}
\newcommand{\SlabInflLagFifty}{0.0}
\newcommand{\StabInflBrkFifty}{0.1}
\newcommand{\StabNervBrkFifty}{6.8}
\newcommand{\SlabInflBrkFifty}{0.2}
\newcommand{\SlabNervBrkFifty}{401.5}
\newcommand{\HstarVsMaLossTen}{0.9}
\newcommand{\CpSixteenEighty}{615.6}

\begin{document}

\title{SLAB: Time-Lag-Aware Deep Reinforcement Learning for Flexible
Job-Shop Scheduling in PPVC Module Factories}

\author{Ziheng~Zhang,~\IEEEmembership{Student~Member,~IEEE,}
Dhivya~Dharshini~Kannan,
and~Wei~Zhang,~\IEEEmembership{Senior~Member,~IEEE}%
\thanks{Z.~Zhang (Research Fellow), D.~D.~Kannan (Research Engineer), and
W.~Zhang (Associate Professor) are with the Singapore Institute of
Technology, Singapore (e-mail: ziheng.zhang@singaporetech.edu.sg;
\mbox{DhivyaDharshini.Kannan@singaporetech.edu.sg};
wei.zhang@singaporetech.edu.sg).
\emph{(Corresponding author: Wei~Zhang.)}}%
\thanks{Manuscript submitted July~2026; revised October~2026. This work has been submitted to the IEEE for possible publication. Copyright may be transferred without notice, after which this version may no longer be accessible.}}
\markboth{IEEE Transactions on Industrial Informatics}%
{Zhang \emph{et al.}: Time-Lag-Aware DRL for FJSP in PPVC Module Factories}

\maketitle

\begin{abstract}
Prefabricated prefinished volumetric construction (PPVC) builds
room-sized modules in factories that operate as flexible job shops. A
module waits \DryLo{} to \CureHi{}\,h after casting, waterproofing, and
painting, while its station is free for other modules. Existing learned
dispatchers release the next operation before this wait ends, so their
plans need repair. This paper
presents SLAB (Scheduling under time-Lags with an Admissible Bound), a
decision process for dispatchers trained by deep reinforcement learning.
Its transition releases a module only after its wait, and its reward is
the decrease of a proven lower bound that includes the waits. Training
on \BatchSmall-module orders with halved waits also teaches the policy
the station contention of larger orders. On three existing networks,
repaired plans are \RepairPenLo\% to \RepairPenHi\% longer than the same
network's schedules inside SLAB. On orders generated from Singapore's PPVC
guidebook, one trained policy gives schedules
\HstarVsSptGainTen\% to \HstarVsSptGainEighty\% shorter than the
shortest-processing-time rule at \BatchSmall{} to \BatchProdHi{}
modules. At \BatchProdLo{} and \BatchProdHi{} modules, it outperforms a
time-lag memetic algorithm by \HstarVsMaGainFifty\% and
\HstarVsMaGainEighty\% and constraint-programming search started
from that rule's schedule and given
\AnytimeBudget{}\,s by \HstarVsCpWarmSptGainFifty\% and
\HstarVsCpWarmSptGainEighty\%. The memetic algorithm gives shorter
schedules on orders of \BatchSmall{} modules, and the same search given
\AnytimeBudgetMid{}\,s on orders of \BatchSmall{} and 20 modules. A
module factory can therefore train one dispatcher on small orders and
run it without a solver on production batches of \BatchProdLo{} to
\BatchProdHi{} modules.
\end{abstract}


\section{Introduction}\label{sec:intro}
\IEEEPARstart{P}{refabricated} prefinished volumetric construction
(PPVC) moves most building work from the site into a factory that
produces room-sized modules~\cite{hammad2020novel,peiris2023production}.
Singapore's Building and Construction Authority (BCA) requires PPVC on
selected government land-sale sites and publishes a guidebook of its
production steps~\cite{bca2017ppvc}. One documented PPVC development
erected \AnchorTowers{} \AnchorFloors-story towers from
\AnchorModules{} modules, or about \ModulesPerFloor{} per
floor~\cite{unitedtec2022avenuesouth}. Such towers therefore rise only
as fast as the module factory delivers modules. The factory routes each
module through stations in an order that depends on the module's class,
and several stations of the same type can serve each visit. A planner
thus decides which station processes each operation and when it starts,
which is the flexible job-shop scheduling problem (FJSP).

\begin{figure*}[!t]
  \centering
  \includegraphics[width=\textwidth]{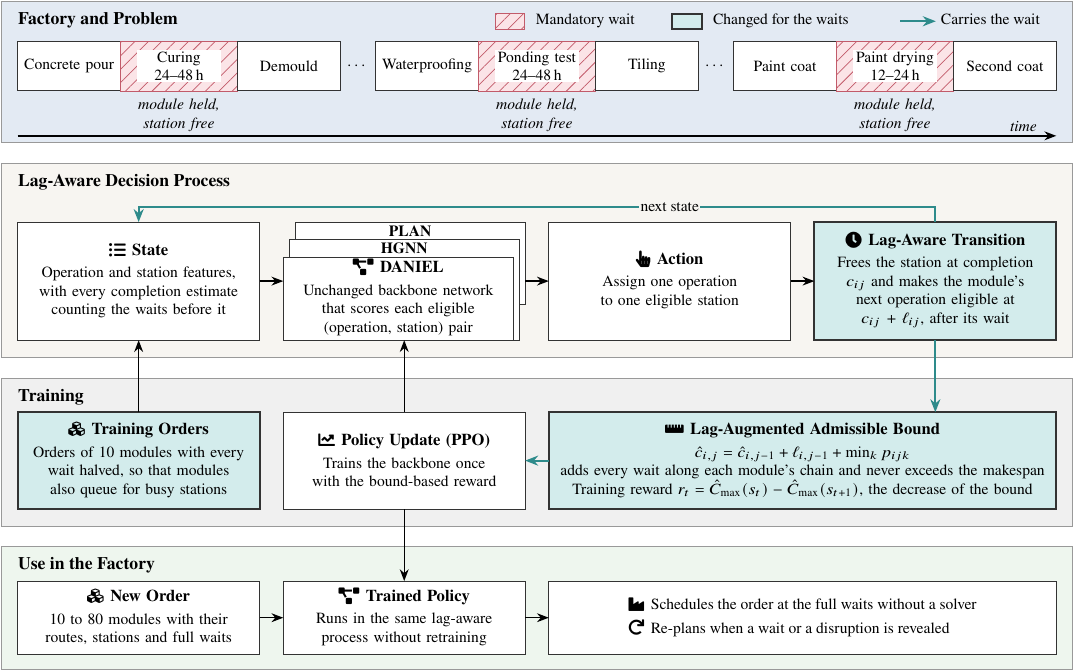}
  \caption{Overview of the proposed lag-aware decision process. The
  dispatch loop places each wait in the transition and in the lower
  bound, and the training branch uses \BatchSmall-module orders with
  halved waits.}
  \label{fig:overview}
\end{figure*}

The PPVC shop differs from the classical FJSP because three kinds of
operation are followed by a wait. Cast concrete cures for
\CureLo--\CureHi{}\,h, a waterproofed wet module holds water in a
ponding test for \PondLo--\PondHi{}\,h, and each paint coat dries for
\DryLo--\DryHi{}\,h. The module is held in a buffer area during each
wait, while the station that processed it is free to serve other
modules. Scheduling models call such a wait a job-blocking, machine-free
time-lag. The waits of an order built from the guidebook's routes add up to \LagRatioMean{} times its total processing time on average. A plan
built without them must be shifted afterward, and each shift also delays
the modules queued behind the shifted one. Queuing for a busy station is
the second source of module idle time. It accounts for only
\ContTenLagFullPct\% of module idle time in \BatchSmall-module orders
scheduled by the shortest-processing-time (SPT) rule. The share rises to
\ContFiftyPct\% and \ContEightyPct\% in production batches of \BatchProdLo{}
and \BatchProdHi{} modules, which supply two to three floors. A
scheduler for this factory must therefore know when each module becomes
ready again and resolve the station queues of production batches.

Three families of methods meet parts of these requirements.
Constraint-programming solvers such as CP-SAT~\cite{perron2024cpsat}
model each wait exactly and find optimal or near-optimal schedules for
small orders, but their search time grows quickly with the order size.
Precast-plant
schedulers~\cite{chan2002production,wang2023production,kim2022rlprecast,zu2025rlprecast}
and time-lag metaheuristics for job
shops~\cite{caumond2008memetic,gonzalez2015scatter,aallaoui2026graspxvns}
also model the wait, but they fix the route of each module or search
anew for every order. Constructive dispatchers trained by deep
reinforcement learning (DRL) build a schedule one operation--station
assignment at a time within seconds. A policy trained on small orders
also schedules larger
ones~\cite{zhang2020l2d,song2023fjsp,wang2024daniel,lei2022multiaction,lei2024largescale,kannan2026plan}.
The two closest dispatchers are DANIEL~\cite{wang2024daniel} and
PLAN~\cite{kannan2026plan}, which reward each step by the decrease of a
makespan lower bound that chains minimum processing times. Like the rest
of the family, they make a module's next operation eligible as soon as
its current operation completes. Under a post-operation wait this
release comes too early, so their plans must be repaired. On three existing networks, the repaired schedules are \RepairPenLo\% to
\RepairPenHi\% longer than those the same network builds when it
dispatches with the waits respected (Table~\ref{tab:necessity}).
Training with the waits raises a second difficulty, because these
dispatchers learn on small orders that the waits dominate. A policy
trained on \BatchSmall-module orders at the true waits therefore learns
little about the station contention of production batches. What remains
missing is a decision process that puts the wait where the dispatcher
acts, together with training orders that still teach station contention.

The key idea is to put the wait into the decision process that trains
and executes the dispatcher, and to shorten the waits only in the
training orders. With the wait in the state transition, every decision
uses the true ready time of each module, so the dispatcher builds
schedules that the factory can run as built. The pending waits also
enter the lower bound whose decrease defines the reward, and this bound
still never exceeds the makespan of a feasible completion. Shortened
training waits raise the share of idle time spent queuing, so a policy
trained on small orders also learns the station contention of larger
ones. Prior dispatchers keep the decision process of the classical FJSP
and improve the network, while this idea changes the process and the
training orders and keeps the network.

We implement this idea as SLAB (Scheduling under time-Lags with an
Admissible Bound), a lag-aware decision process for constructive
dispatching (Fig.~\ref{fig:overview}). Its transition frees a station
when an operation completes and makes the module's next operation
eligible only after the wait. The reward of a step is the decrease of a
lag-augmented lower bound, which adds each pending wait to the chain of
minimum processing times of its module. SLAB trains the policy on orders of
\BatchSmall{} modules whose waits are all halved and dispatches it at
the true waits. Halving the waits raises the share of idle time spent
queuing from \ContTenLagFull{} to \ContTenLagHalf{}, close to the
\ContTwenty{} of a 20-module order at the true waits. Both parts act
outside the network, so any backbone that scores the eligible
operation--station pairs of a state can use them. DANIEL, the
heterogeneous graph neural network (HGNN)~\cite{song2023fjsp}, and PLAN
meet this condition. Each backbone is trained by
proximal policy optimization (PPO)~\cite{schulman2017ppo} on orders from
a generator built on the BCA guidebook. The policies are tested on
orders of \BatchSmall{} to \BatchProdHi{} modules on \FactoryDefault{}
stations, including independently generated orders at four sizes. Their
comparators are \nRules{} dispatching rules, \nLagRules{} of which
observe the waits, a time-lag memetic
algorithm~\cite{caumond2008memetic}, and CP-SAT started from scratch or
from a given schedule.

The SLAB policy trained on halved waits builds the shortest schedules
among the methods that need no solver from 20 to \BatchProdHi{} modules
(Table~\ref{tab:main}). Its mean makespan over \NumSeeds{} training
seeds is \HstarVsSptGainTen\% to \HstarVsSptGainEighty\% shorter than
that of the SPT rule from \BatchSmall{} to \BatchProdHi{} modules. This
margin grows with the order size and is \HstarIndVsSptGainTen\% to
\HstarIndVsSptGainEighty\% on the independently generated orders. At
\BatchProdLo{} and \BatchProdHi{} modules, its makespan is
\HstarVsMaGainFifty\% and \HstarVsMaGainEighty\% shorter than that of
the memetic algorithm given \MaBudget{}\,s per order. Its margin over
CP-SAT started from the SPT schedule and given \AnytimeBudget{}\,s on
\AnytimeCores{} cores is \HstarVsCpWarmSptGainFifty\% and
\HstarVsCpWarmSptGainEighty\% (Fig.~\ref{fig:anytime}). On orders of
\BatchSmall{} modules, the policy is \HstarVsMaLossTen\% longer than the
memetic algorithm and \HstarVsCpWarmSptTen\% longer than CP-SAT given
\AnytimeBudgetMid{}\,s, and on orders of 20 modules it stays
\HstarVsCpWarmSptTwenty\% longer than CP-SAT. From 30 modules on, it is
shorter than every method that does not start from its schedule. The
halved training waits let one policy match
or lead the full-wait and wait-free training options at both ends of the
size range (Table~\ref{tab:trainorders}). Full training waits make the same network
\FullVsHalfFifty\% and \FullVsHalfEighty\% longer at
\BatchProdLo{} and \BatchProdHi{} modules, where queuing dominates. A
network trained without waits matches the policy there but is longer at
\BatchSmall{} and 20 modules, where the waits weigh more. PLAN trained on
the same orders shows the same pattern.

The policy schedules an order of \BatchProdHi{} modules in
\SlabLatencyEighty\,s on one CPU core, and re-planning keeps its
schedules short under uncertain waits and disruptions. When each wait
varies with a coefficient of variation up to \StochCvHi{}, re-planning
keeps the policy within \StochGapPiMax\% of its makespan under known
waits (Table~\ref{tab:uncertain}). After five of \nDisruptTypes{}
disruption types on \BatchProdLo-module orders, re-planning with the
policy gives the lowest mean makespan over 20 orders, significantly
below re-planning by rules or a genetic algorithm and below right-shift
repair (Table~\ref{tab:disrupt}). Residual CP-SAT given 30\,s ends
significantly above the policy
after an urgent module, a material delay, or a cancellation and below it
after a short station breakdown. The policy needs no
retraining when every wait is \FtLagScale{} times longer, and
\FtUpdates{} updates of fine-tuning adapt it to a factory with
\FactoryTight{} stations (Table~\ref{tab:deploy}).

These findings identify which scheduler a module factory should run at
each order size. A factory can train one dispatcher once on small
orders with halved waits and run it without a solver on orders of 30 to
\BatchProdHi{} modules, which include production batches. When waits
dominate small training orders, researchers can shorten the waits until
queuing takes a third to a half of idle time. Both changes lie outside
the network, so learned schedulers for other shops in which a job waits
while its machine is free can adopt them. The transition carried over to
all three tested networks, and the training rule to PLAN but not to
HGNN. This paper makes four contributions.
\begin{enumerate}
\item \textbf{Problem and evidence.} A formulation of module-factory
scheduling as the FJSP with job-blocking, machine-free post-operation
time-lags (FJSP-TL). On three existing networks, repairing a plan built
without the waits lengthens it by \RepairPenLo\% to \RepairPenHi\%.
\item \textbf{SLAB.} A lag-aware transition and a lower bound that
includes the waits. Proposition~\ref{prop:lb} proves the bound
admissible and tight at termination, so maximizing the return minimizes
the makespan. Both run unchanged inside DANIEL, HGNN, and PLAN after
each network reproduces its reported results. The SLAB policy is
\HstarVsSptGainTen\% to \HstarVsSptGainEighty\% shorter than SPT at
\BatchSmall{} to \BatchProdHi{} modules. At \BatchProdLo{} and
\BatchProdHi{} modules, it outperforms the memetic algorithm and
CP-SAT started from SPT and given \AnytimeBudget{}\,s.
\item \textbf{Training orders.} A rule that halves the waits of
\BatchSmall-module training orders, so that queuing takes the share of
idle time it has in a 20-module order. The resulting policy
matches or leads those trained at the full waits and without waits at
every size of the default factory (Table~\ref{tab:trainorders}), and the
rule carries over to PLAN.
\item \textbf{Benchmark and deployment evidence.} A public order
generator built on the BCA guidebook, test sets of \BatchSmall{} to
\BatchProdHi{} modules with independently generated sets, and the
trained policies. Their evaluation covers uncertain waits,
\nDisruptTypes{} disruption types, finite buffers, variable routes, and
fine-tuning to a new factory layout.
\end{enumerate}

\FloatBarrier

\section{Related Work}\label{sec:related}
Constructive DRL schedulers for the FJSP build a schedule one
operation--station assignment at a time and reward each step by the
decrease of a makespan estimate. Zhang \emph{et al.}~\cite{zhang2020l2d}
introduced this scheme for the job shop with a lower bound that chains
the remaining processing times of each job, and Lei
\emph{et al.}~\cite{lei2022multiaction} carried it to the FJSP.
DANIEL~\cite{wang2024daniel} uses the chain bound, computed with the
minimum processing time over compatible stations, as reward and as state
feature, and PLAN~\cite{kannan2026plan} keeps this environment with a
lighter encoder. HGNN~\cite{song2023fjsp} computes the chain from mean
processing times, so its estimate is not a lower bound. All these
schedulers, and the dynamic job-shop dispatcher of Liu
\emph{et al.}~\cite{liu2024dynamicgat}, assume that a job is ready when its
predecessor completes. Luo \emph{et al.}~\cite{luo2022realtime}
schedule a partial no-wait FJSP, in which some operations start
right after their predecessor. A post-operation lag is a positive gap
with the station free. The readiness rule then
releases the successor too early, so the return measures the makespan
of a schedule the factory cannot execute.

Stochastic extensions randomize the processing times and keep this
readiness assumption. Smit \emph{et al.}~\cite{smit2025stochastic} track
the chain bound in sampled processing-time scenarios and reward the
change in a high quantile (value-at-risk) or in the mean of the scenario
bounds.
PLAN~\cite{kannan2026plan} adopts this scenario module and also tests
job arrivals and machine breakdowns. None of these models includes a
random wait between operations. A random lag cannot be recast as a
random processing time either, because a longer operation keeps its
station busy while a lag leaves the station free.

Lag-aware scheduling models the lag exactly but fixes the route or
searches offline. Precast plant scheduling models curing as a delay in a
fixed-route flow shop and solves it by
metaheuristics~\cite{chan2002production,wang2023production} or by
reinforcement learning~\cite{kim2022rlprecast,zu2025rlprecast}. A PPVC
route instead revisits the quality-check stations, and the four routing
classes share parallel stations of each type. The flexible job shop
models both features. Against fixed station lines, flexible routing
shortens 10-module schedules by \FlowGainSptTen\% (SPT) and
\FlowGainCpTen\% (CP-SAT) (Section~S6 of the Supplementary Material).
Offline time-lag search began with the job
shop. Caumond \emph{et al.}~\cite{caumond2008memetic} solve the job shop with minimal and maximal lags by a memetic algorithm.
Gonz\'alez \emph{et al.}~\cite{gonzalez2015scatter}
add setup times to this problem, and Aallaoui
\emph{et al.}~\cite{aallaoui2026graspxvns} extend it to the flexible job
shop with generic minimal lags. These methods search anew for every
instance and train no dispatcher. This paper places the lag in the transition and in the bound of the
decision process, so that a learned dispatcher trains on the shop it
will schedule. To our knowledge, no constructive DRL formulation of the
flexible job shop with job-blocking, machine-free time-lags has been
reported before.

\section{Problem: Flexible Job Shop With Post-Operation Time-Lags}
\label{sec:problem}
\subsection{Definition}
\begin{table}[!t]
\caption{Main Notation.}
\label{tab:notation}
\centering
\begin{tabular}{ll}
\toprule
Symbol & Description \\
\midrule
$J_i$            & module (job) $i$ \\
$M_k$            & station (machine) $k$ \\
$O_{ij}$         & $j$-th operation of module $i$ \\
$n$, $m$         & number of modules, stations \\
$n_i$            & number of operations of module $i$ \\
$\mathcal{M}_{ij}$ & indices of stations eligible for $O_{ij}$ \\
$p_{ijk}$        & processing time of $O_{ij}$ on $M_k$ \\
$\ell_{ij}$      & post-operation time-lag of $O_{ij}$ \\
$s_{ij}$         & start time of $O_{ij}$ \\
$c_{ij}$         & completion time, $c_{ij}=s_{ij}+p_{ij,k_{ij}}$ \\
$k_{ij}$         & station assigned to $O_{ij}$ \\
$C_{\max}$       & factory makespan, $\max_i c_{i,n_i}$ \\
$\hat{c}_{i,j}$  & lag-augmented completion lower bound \\
$\rho_i$         & job ready time, $c_{ij}+\ell_{ij}$ \\
$\eta^{\text{lag}}_{ij}$ & static lag channel, $=\ell_{ij}$ \\
$\eta^{\text{rem}}_{ij}$ & remaining in-flight lag at $t$ \\
$\tau_{ij}$, $\tau_k$ & operation / station type index \\
$E_{\text{op}}$, $E_{\text{mch}}$ & op / station type embeddings \\
$r_t$            & PPO reward, decrease of $\hat{C}_{\max}$ \\
\bottomrule
\end{tabular}
\end{table}

An instance has $n$ modules (jobs) $\mathcal{J}=\{J_1,\dots,J_n\}$
and $m$ stations (machines) $\mathcal{M}=\{M_1,\dots,M_m\}$, and
Table~\ref{tab:notation} summarizes the notation. Module $J_i$ is
fabricated by an ordered route
$O_{i,1}\prec O_{i,2}\prec\dots\prec O_{i,n_i}$. Operation $O_{ij}$ may be processed
by any station in a compatible index set $\mathcal{M}_{ij}\subseteq\{1,\dots,m\}$, with processing time
$p_{ijk}$ on station $M_k$, $k\in\mathcal{M}_{ij}$. Operation
$O_{ij}$ also imposes a \emph{post-operation time-lag} $\ell_{ij}\ge 0$
on its module after it completes. Curing
after casting, ponding after waterproofing, and drying after each paint
coat are such lags, and the other operations of a PPVC route have
$\ell_{ij}=0$. Each operation further carries a type index $\tau_{ij}$
(\nOpTypes{} classes) and each station a type $\tau_k$
(\nStationTypes{} classes), and Section~\ref{sec:benchmark} derives both
taxonomies from the BCA guidebook. A scheduler may read these types as
input, but no constraint uses them.

A schedule assigns each operation a station
$k_{ij}\in\mathcal{M}_{ij}$ and a start time $s_{ij}\ge 0$, with
completion $c_{ij}=s_{ij}+p_{ij,k_{ij}}$, subject to
\begin{align}
s_{i,j+1} &\;\ge\; c_{ij} + \ell_{ij}, \qquad \forall i,\ 1\le j<n_i,
  \label{eq:lagprec}\\
[s_{ij},c_{ij}) &\cap [s_{i'j'},c_{i'j'}) = \emptyset \nonumber\\
 &\qquad \text{whenever } k_{ij}=k_{i'j'},\ (i,j)\ne(i',j').
  \label{eq:capacity}
\end{align}
Constraint~\eqref{eq:lagprec} is a finish-start precedence with a
minimal time-lag. The next operation of a module can start only after
the lag of its previous operation has elapsed, so a freshly cast module
stays out of production until its concrete has cured.
Constraint~\eqref{eq:capacity} lets a station process one operation at
a time, and an operation occupies its station only during
$[s_{ij},c_{ij})$. The lag therefore holds the module in a buffer area
while its station serves other modules. This job-blocking, machine-free behavior separates
FJSP-TL from a model that adds the lag to the processing time. Such a
model blocks the station for $p+\ell$ and overstates congestion and
makespan (Fig.~\ref{fig:gantt}). A sequence-dependent setup differs as
well, because it belongs to a machine and to the order of jobs on that
machine. The objective is the factory makespan
$C_{\max}=\max_{i} c_{i,n_i}$, the time at which the last module
completes its last operation.\footnote{The terminal lag
$\ell_{i,n_i}$ delays module shipment but occupies no factory resource,
so we report the station-side makespan. In our benchmark every route
ends with final quality-gate and wrap-and-ship operations with
$\ell_{i,n_i}=0$, so the two definitions coincide.} Minimizing
$C_{\max}$ is strongly NP-hard, because setting $\ell\equiv 0$ and
$|\mathcal{M}_{ij}|=1$ recovers the classical job-shop makespan problem,
which is strongly NP-hard~\cite{garey1976complexity}.

\begin{figure}[!t]
  \centering
  \resizebox{\columnwidth}{!}{\begin{tikzpicture}[
  font=\scriptsize,
  >={Stealth[length=1.8mm,width=1.4mm]},
  x=4.3mm, y=1mm,
  proc/.style={draw, line width=0.5pt, fill=black!15},
  lagbar/.style={draw, line width=0.5pt,
                 pattern=north east lines, pattern color=black!65},
  axis/.style={->, line width=0.5pt},
  tick/.style={line width=0.4pt},
  cmax/.style={dashed, line width=0.7pt},
  rowlab/.style={anchor=east, font=\scriptsize, inner sep=1.5pt},
  blab/.style={font=\scriptsize, inner sep=0pt},
  panttl/.style={font=\scriptsize\bfseries, anchor=west, inner sep=0pt},
  note/.style={font=\scriptsize, inner sep=0pt},
]

\def\ph{6}
\def\gp{4}
\def\rowx{0}
\def\axlen{10}

\def\aMone{0}
\def\aMtwo{\aMone-\ph-\gp}

\fill[proc] (0,\aMone) rectangle (3,\aMone+\ph);
\draw       (0,\aMone) rectangle (3,\aMone+\ph);
\node[blab] at (1.5,\aMone+\ph/2) {$O_{11}$};
\fill[lagbar] (3,\aMone) rectangle (6,\aMone+\ph);
\draw         (3,\aMone) rectangle (6,\aMone+\ph);
\node[blab] at (4.5,\aMone+\ph/2) {$\ell$};
\fill[proc] (6,\aMone) rectangle (9,\aMone+\ph);
\draw       (6,\aMone) rectangle (9,\aMone+\ph);
\node[blab] at (7.5,\aMone+\ph/2) {$O_{21}$};

\node[rowlab] at (\rowx,\aMone+\ph/2) {$M_1$};
\node[rowlab] at (\rowx,\aMtwo+\ph/2) {$M_2$};

\node[panttl, anchor=south west] at (\rowx-0.2,\aMone+\ph+6.5)
  {(a) Processing-time inflation $p_{11}\!+\!\ell$: station blocked};

\draw[cmax] (9,\aMtwo-2) -- (9,\aMone+\ph+1.0);
\node[note, anchor=south west] at (9.1,\aMone+\ph+0.2) {$C_{\max}^{(a)}$};

\draw[axis] (\rowx,\aMtwo-3) -- (\axlen,\aMtwo-3);
\foreach \t in {0,3,6,9}{
  \draw[tick] (\t,\aMtwo-3) -- (\t,\aMtwo-4.2);
  \node[note,anchor=north] at (\t,\aMtwo-4.4) {\t};
}
\node[note,anchor=west] at (\axlen,\aMtwo-3) {\,$t$};

\def\boff{-34}
\def\bMone{\aMone+\boff}
\def\bMtwo{\bMone-\ph-\gp}
\def\bMod{\bMtwo-\ph-\gp}

\fill[proc] (0,\bMone) rectangle (3,\bMone+\ph);
\draw       (0,\bMone) rectangle (3,\bMone+\ph);
\node[blab] at (1.5,\bMone+\ph/2) {$O_{11}$};
\fill[proc] (3,\bMone) rectangle (6,\bMone+\ph);
\draw       (3,\bMone) rectangle (6,\bMone+\ph);
\node[blab] at (4.5,\bMone+\ph/2) {$O_{21}$};

\fill[proc] (6,\bMtwo) rectangle (8,\bMtwo+\ph);
\draw       (6,\bMtwo) rectangle (8,\bMtwo+\ph);
\node[blab] at (7,\bMtwo+\ph/2) {$O_{12}$};

\fill[lagbar] (3,\bMod) rectangle (6,\bMod+\ph);
\draw         (3,\bMod) rectangle (6,\bMod+\ph);
\node[blab] at (4.5,\bMod+\ph/2) {$\ell$};

\node[rowlab] at (\rowx,\bMone+\ph/2) {$M_1$};
\node[rowlab] at (\rowx,\bMtwo+\ph/2) {$M_2$};
\node[rowlab, align=right] at (\rowx,\bMod+\ph/2) {module 1\\(buffer)};

\node[panttl, anchor=south west] at (\rowx-0.2,\bMone+\ph+6.5)
  {(b) Job-blocking, machine-free (ours): station free during lag};

\node[note, anchor=south east] at (2.9,\bMone+\ph+1.4) {$M_1$ free at $c_{11}$\,};
\draw[->,line width=0.4pt] (3,\bMone+\ph+1.4) -- (3,\bMone+\ph);
\draw[cmax] (6,\bMod-2) -- (6,\bMone+\ph+1.0);
\node[note, anchor=south east] at (5.9,\bMone+\ph+0.2) {$C_{\max}^{(b)}$};

\draw[cmax, black!45] (9,\bMod-2) -- (9,\bMone+\ph+1.0);
\node[note, anchor=south west, black!45] at (9.1,\bMone+\ph+0.2) {$C_{\max}^{(a)}$};
\draw[<->,line width=0.4pt] (6,\bMod-1) -- (9,\bMod-1);
\node[note,anchor=north] at (7.5,\bMod-1) {saving};

\draw[axis] (\rowx,\bMod-3) -- (\axlen,\bMod-3);
\foreach \t in {0,3,6,9}{
  \draw[tick] (\t,\bMod-3) -- (\t,\bMod-4.2);
  \node[note,anchor=north] at (\t,\bMod-4.4) {\t};
}
\node[note,anchor=west] at (\axlen,\bMod-3) {\,$t$};

\def\leg{\bMod-11}
\fill[proc] (\rowx,\leg) rectangle (\rowx+1.6,\leg+\ph*0.7);
\draw       (\rowx,\leg) rectangle (\rowx+1.6,\leg+\ph*0.7);
\node[note,anchor=west] at (\rowx+2.0,\leg+\ph*0.35) {processing $p$};
\fill[lagbar] (\rowx+5.6,\leg) rectangle (\rowx+7.2,\leg+\ph*0.7);
\draw         (\rowx+5.6,\leg) rectangle (\rowx+7.2,\leg+\ph*0.7);
\node[note,anchor=west] at (\rowx+7.6,\leg+\ph*0.35) {lag $\ell$ (buffer)};

\end{tikzpicture}}
  \caption{Two models of a post-operation time-lag $\ell$ on a toy
  instance with two modules and two stations. (a) Adding $\ell$ to the
  processing time of $O_{11}$ blocks station $M_1$ until the lag ends.
  (b) Under~\eqref{eq:capacity}, $M_1$ is free at $c_{11}$ and serves
  $O_{21}$, while module~1 waits in a buffer (hatched) and $O_{12}$
  starts at $c_{11}\!+\!\ell$.}
  \label{fig:gantt}
\end{figure}

The lag also changes the disjunctive graph. Once
$O_{ij}$ is assigned to station $k_{ij}$, its outgoing job arc carries
weight $p_{ij,k_{ij}}+\ell_{ij}$, while its outgoing machine arcs carry
weight $p_{ij,k_{ij}}$. The makespan of a complete schedule is the
longest path in this graph, so this asymmetry carries over to every
makespan estimate used during search or learning.

\subsection{Why Lags Change the Problem}
Modeling the lags lengthens even the optimal schedule
(Table~\ref{tab:inflation}). On one \MotModules-module order on
\FactoryDefault{} stations, the optimal makespan computed by
CP-SAT~\cite{perron2024cpsat} rises from \MotOptFree{}\,h to
\MotOptLag{}\,h (+\MotOptInfl\%). Over the \nTest{} test
orders of \BatchSmall{} modules, the lags raise the makespan found by
CP-SAT by \TestLagInflation\% on average. The lags lengthen only job
arcs, and the longest path of the lagged optimum runs through waits in
which the module occupies no station. Delaying the operations of a plan
built without these waits restores feasibility but increases makespan. The
lag-free plan of the first-in-first-out (FIFO) rule in
Table~\ref{tab:inflation} (\MotFifoFree{}\,h), shifted
later until every lag holds, ends at \MotRepaired{}\,h. This makespan is
\MotRepairPenalty\% longer than the \MotFifoLag{}\,h that FIFO reaches
when it observes the lags. A scheduler
must therefore know when each module becomes ready.

\begin{table}[!t]
\caption{Makespan of a Five-Module Order on \FactoryDefault{} Stations
Without and With Time-Lags}
\label{tab:inflation}
\centering
\begin{tabular}{lccc}
\toprule
Scheduler & Lag-Free (h) & Lag-Aware (h) & Increase \\
\midrule
FIFO rule & \MotFifoFree & \MotFifoLag & $+\MotFifoInfl\%$ \\
CP-SAT optimum & \MotOptFree & \MotOptLag & $+\MotOptInfl\%$ \\
\bottomrule
\end{tabular}
\end{table}

\section{SLAB: A Lag-Aware Decision Process for Constructive Scheduling}
\label{sec:method}
SLAB (Scheduling under time-Lags with an Admissible Bound) is a
decision process for training and running a constructive DRL scheduler
on FJSP-TL instances. It includes each post-operation lag in the state transition
and the admissible bound that defines the reward. Two optional input features, called input aids below, describe the waits and the operation and station types to the network. Fig.~\ref{fig:arch} and
Algorithm~\ref{alg:rollout} show one step and one episode with the DANIEL
backbone~\cite{wang2024daniel}.

\begin{figure}[!t]
  \centering
  \resizebox{\columnwidth}{!}{\begin{tikzpicture}[
  font=\scriptsize,
  >={Stealth[length=2mm,width=1.5mm]},
  node distance=4.2mm,
  box/.style={draw, line width=0.6pt, rounded corners=1.5pt,
              align=center, minimum width=13.5mm, minimum height=9mm,
              fill=white, inner sep=1pt},
  adapt/.style={draw, line width=0.6pt, fill=black!12,
                rounded corners=1.5pt, align=center, inner sep=1pt,
                minimum width=14.5mm, minimum height=10mm, text width=13mm},
  flow/.style={->, line width=0.7pt},
  tap/.style={->, dashed, line width=0.5pt, shorten >=1pt},
]
\node[box] (env) {FJSP-TL\\environment};
\node[box, right=of env]  (feat) {state\\features};
\node[box, right=of feat] (norm) {per-step\\$z$-norm};
\node[box, right=of norm] (attn) {dual\\attention};
\node[box, right=of attn] (pool) {nonzero\\pooling};
\node[box, right=of pool] (pol)  {policy\\$\pi_\theta$};
\foreach \a/\b in {env/feat,feat/norm,norm/attn,attn/pool,pool/pol}
  {\draw[flow] (\a)--(\b);}

\node[draw=black!55, dash pattern=on 1.6pt off 1.4pt, line width=0.6pt,
      rounded corners=3pt, inner sep=3.4mm, fit=(env)(pol)] (bb) {};
\node[anchor=north, font=\scriptsize\itshape, black!60, fill=white,
      inner sep=0.6pt]
      at ($(bb.south)+(0,-1.1mm)$) {DANIEL backbone (unchanged)};

\node[adapt, above=7mm of env]  (a1) {lag-aware\\transition};
\node[adapt, above=7mm of feat] (a2) {optional\\wait\\channels};
\node[adapt, above=7mm of norm] (a3) {optional\\type\\embeddings};
\draw[tap] (a1) -- (env.north);
\draw[tap] (a2) -- (feat.north);
\draw[tap] (a3) -- (norm.north);

\coordinate (rlo) at ($(env.south)+(0,-9.5mm)$);
\coordinate (rro) at ($(pol.south)+(0,-9.5mm)$);
\draw[flow, rounded corners=2.5pt] (pol.south) |- (rro) -- (rlo) |- (env.south);
\node[font=\scriptsize, anchor=north] at ($(rro)!0.5!(rlo)+(0,-0.4mm)$)
  {scheduled $(O_{ij},M_k)$;\; PPO reward
   $r_t=\hat{C}_{\max}(s_t)-\hat{C}_{\max}(s_{t+1})$\; (admissible)};
\end{tikzpicture}}
  \caption{One decision step of SLAB with the DANIEL backbone. White
  boxes are the unchanged backbone. Shaded boxes mark the lag-aware transition and the two optional input aids, and the bottom line gives the reward from the bound.}
  \label{fig:arch}
\end{figure}
\begin{algorithm}[!t]
\caption{One SLAB episode with the DANIEL backbone}
\label{alg:rollout}
\begin{algorithmic}[1]
\Require routes $O_{i,1}\!\prec\!\dots\!\prec\!O_{i,n_i}$; eligible sets
  $\mathcal{M}_{ij}$; times $p_{ijk}$; lags $\ell_{ij}$;
  type indices $\tau_{ij}$ (op), $\tau_k$ (station); policy $\pi_\theta$
\Ensure station $k_{ij}$ and start $s_{ij}$ for all $O_{ij}$; makespan $C_{\max}$
\State $t\!\gets\!0$; all $O_{ij}$ unscheduled;
       $\rho_i\!\gets\!0$, free$_k\!\gets\!0$ $\;\forall i,k$
\State seed bound $\hat{c}_{i,j}$ from \eqref{eq:lb} \Comment{lag-augmented}
\While{some operation is unscheduled}
  \State build op features (10 ch.) and station features
  \State append $\eta^{\text{lag}}_{ij}\!=\!\ell_{ij}$,\;
         $\eta^{\text{rem}}_{ij}(t)$ \Comment{optional wait channels}
  \State $z$-normalize every channel; zero deleted (finished) rows
  \State $\mathbf{e}_{ij}\!\gets\!E_{\text{op}}(\tau_{ij})$,\;
         $\mathbf{e}_{k}\!\gets\!E_{\text{mch}}(\tau_k)$
  \State concat $\mathbf{e}$ after norm., masked by liveness
         $\mathds{1}[\text{row}\!\neq\!0]$ \Comment{optional type embeddings}
  \State dual attention $\rightarrow$ $\pi_\theta$ over eligible
         $(O_{ij},M_k)$, $k\!\in\!\mathcal{M}_{ij}$
  \State pick pair $(O_{ij},M_k)$: greedy $\arg\max$ or sampled
  \State $s_{ij}\!\gets\!\max(\rho_i,\text{free}_k)$;\;
         $c_{ij}\!\gets\!s_{ij}\!+\!p_{ijk}$; set $k_{ij}\!\gets\!k$
  \State free$_k\!\gets\!c_{ij}$;\;
         $\rho_i\!\gets\!c_{ij}\!+\!\ell_{ij}$
         \Comment{free station, hold module}
  \State advance $t$; recompute $\hat{c}_{i,j}$ and $\hat{C}_{\max}$
         via \eqref{eq:lb}
  \If{training}
    \State $r_t\!\gets\!\hat{C}_{\max}(s_t)-\hat{C}_{\max}(s_{t+1})$
           \Comment{admissible reward}
  \EndIf
\EndWhile
\State \Return $\{k_{ij},s_{ij}\}$ and $C_{\max}=\max_i c_{i,n_i}$
\end{algorithmic}
\end{algorithm}

\subsection{Lag-Aware Transition}\label{sec:dynamics}
The lag changes only the update that follows each assignment. When the
policy assigns operation $O_{ij}$ to station $M_k$, the operation
starts once both the module and the station are available, and it
completes at $c_{ij}$. Station $M_k$ is free again at $c_{ij}$, while
module $J_i$ becomes ready for its next operation only at
\begin{equation}
\rho_i \;=\; c_{ij} + \ell_{ij},
\end{equation}
which implements~\eqref{eq:lagprec} and~\eqref{eq:capacity}. The later
ready time raises the earliest start of every pair that contains module
$J_i$, and the decision instant $t$ advances to the earliest start over
all eligible pairs.

\subsection{Lag-Augmented Admissible Bound and Reward}\label{sec:lb}
Constructive DRL solvers use a lower bound on the final makespan as a
dense reward and a state feature. SLAB adds the lags to this bound. For a
partial schedule, SLAB defines the
bound recursively for each job $i$ as
\begin{equation}\label{eq:lb}
\hat{c}_{i,j} \;=\; \hat{c}_{i,j-1} + \ell_{i,j-1}
 + \min_{k\in\mathcal{M}_{ij}} p_{ijk},
\end{equation}
for each unscheduled operation, with the convention
$\hat{c}_{i,0}=\ell_{i,0}=0$. Each scheduled operation seeds the
recursion with its realized completion, $\hat{c}_{i,j}=c_{ij}$, and the
aggregate bound is $\hat{C}_{\max}=\max_i \hat{c}_{i,n_i}$. The
predecessor lag $\ell_{i,j-1}$ enters before the successor's processing
time, so each estimate counts every wait earlier on the module's route.
With $\ell\equiv 0$, \eqref{eq:lb} reduces exactly to the classical
FJSP bound.

\begin{proposition}[Admissibility]\label{prop:lb}
For any partial schedule and any feasible completion of it,
$\hat{C}_{\max}\le C_{\max}$, and the bound is tight at episode
termination, where $\hat{C}_{\max}=C_{\max}$.
\end{proposition}
\begin{proof}
Equation~\eqref{eq:lb} is the makespan of a relaxation of the residual problem. The relaxation (i) drops the station-capacity constraint~\eqref{eq:capacity} and (ii) replaces each processing time by its minimum $\min_{k}p_{ijk}$ over compatible stations, while it retains the exact lag precedence~\eqref{eq:lagprec}. Both relaxations only remove constraints or lower durations, so the optimal makespan cannot increase. Hence $\hat{c}_{i,j}$ lower-bounds the completion of $O_{ij}$ in every feasible extension, and the maximum over terminal operations bounds $C_{\max}$. At termination
every operation is scheduled and seeded with its realized completion
$c_{ij}$, so the relaxation coincides with the realized schedule and
the inequality becomes equality.
\end{proof}

SLAB rewards each assignment with the decrease of this bound,
$r_t=\hat{C}_{\max}(s_t)-\hat{C}_{\max}(s_{t+1})$. These rewards
telescope, and Proposition~\ref{prop:lb} makes the bound equal to the
makespan at termination. The undiscounted return ($\gamma=1$) of an
episode is therefore $\hat{C}_{\max}(s_0)-C_{\max}$. Its first term
depends only on the instance, so maximizing the return minimizes the
makespan. A lag-free bound keeps both properties under the lag-aware transition. The returns under the two
bounds thus differ by an instance constant, and both bounds rank
policies in the same order. They differ only in when a wait enters the
estimate. The lag-augmented bound holds every wait from the first step,
so an assignment lowers the reward only by the delay it adds beyond the
processing times and waits. Rewards under the lag-free bound also
charge waits that no assignment can avoid, because that bound takes in
a wait only when the successor is scheduled.

\subsection{Training Orders With Shortened Waits}\label{sec:trainorders}
The transition and the bound let a policy see every wait, but the
conflicts a policy learns to resolve come from its training orders.
These conflicts show in the idle time of the modules. A module is idle
between two consecutive operations either because it serves its wait
or because it is ready and queues for a busy station. The share of
module idle time spent queuing for a station is
\begin{equation}\label{eq:share}
\sigma=\frac{Q}{Q+W},\quad
Q=\sum_{i}\sum_{j<n_i}\bigl(s_{i,j+1}-c_{ij}-\ell_{ij}\bigr),
\end{equation}
with $W=\sum_{i}\sum_{j<n_i}\ell_{ij}$. Here $Q$ is the total queuing time of
the modules and $W$ their total wait. We compute $\sigma$ on the SPT
schedules of 20 generated orders.

The share differs between small training orders and production batches. A
10-module order at the true waits has $\sigma=\ContTenLagFull$, so the
waits fill most of the idle time and stations are rarely contested.
Production batches of 50 and 80 modules have $\sigma=\ContFifty$ and
$\sigma=\ContEighty$. Most idle time in these batches is
queuing, so the decision that matters most there is which module
receives a contested station. This decision is rare in 10-module
training orders.

SLAB therefore trains on orders whose waits are multiplied by a factor
$\kappa\in(0,1]$ and rounded to whole hours, and it evaluates and
deploys the policy at the true waits. The factor applies to training
and validation orders only, and it changes neither the network nor
PPO. This paper uses $\kappa=0.5$, which raises $\sigma$ of a 10-module
order to \ContTenLagHalf{}. The share then lies between that of a
20-module order at the true waits (\ContTwenty) and that of a 10-module
order on the \FactoryTight-station factory (\ContTight).
Removing the waits ($\kappa=0$) would also expose contention, with
$\sigma=\ContTenLagZero$, but the policy would then never meet a module
that is held after its operation. Halving keeps every wait in the
transition and in the bound with a nonzero length, so the policy meets
both causes of idle time in the same orders. A factory can compute
$\sigma$ for its own layout and choose $\kappa$ before any training,
because $\sigma$ needs only rule schedules of generated orders.

\subsection{Backbone Interface}\label{sec:backbone}
The transition and the bound are defined on the scheduling state
alone, so they do not depend on the network. A backbone can learn on
them if it scores every eligible operation--station pair from the
state and is trained by PPO~\cite{schulman2017ppo} on a
bound-difference reward. Improvement methods that edit a complete
schedule lie outside this interface. The experiments use three
backbones that fit it. DANIEL applies dual attention over operations
and stations. HGNN~\cite{song2023fjsp} embeds operations and stations
as two node types of one graph. PLAN~\cite{kannan2026plan} replaces DANIEL's
attention blocks with attention coupled to a liquid neural network
block and keeps DANIEL's state and pooling. Only HGNN needs an
adjustment, because its own reward estimate uses mean processing times
and can exceed the final makespan even without lags. SLAB replaces this estimate with the bound~\eqref{eq:lb}, which
Proposition~\ref{prop:lb} makes admissible, and keeps every other
hyperparameter of each backbone at its published value.

\subsection{Optional Input Aids and Reproduction}\label{sec:aids}
Two optional aids add input features that describe the waits and the
operations before the policy commits an operation. The first widens the
operation features of DANIEL from 10 to 12 channels. One channel holds
the wait $\ell_{ij}$ that an operation will impose, and the other holds
the part of an ongoing wait that remains at the decision instant. The
second aid embeds the class of each operation and station, out of
\nOpTypes{} operation types and \nStationTypes{} station types. Neither
aid changes the transition
or the bound, and neither shortens the schedules on the default factory
(Section~\ref{sec:mechanism}), so the SLAB policy runs without them.

Each part of SLAB has its own switch. With every switch off and every
wait at zero, the DANIEL code reproduces the published solver exactly,
with a mean makespan of \RegressMean{} on the SD2 instance set of its
paper. The HGNN port reproduces the published mean of \HgnnRegress{} on
the $10\times5$ test set of its authors. With the five checkpoints of
its authors, the PLAN port reproduces their reported mean of
\PlanRegress{} on SD2. A measured
difference between policies therefore comes from the parts of SLAB or
from the training orders, while the ports themselves add none.

\subsection{Scope of the Formulation}\label{sec:generality}
FJSP-TL covers minimal, job-blocking, machine-free lags.
Table~\ref{tab:generality} lists how four neighboring constraint types
would change its decision process. Each type changes the transition that
sets ready and release times, the action mask over operation--station
pairs, or the bound~\eqref{eq:lb}. This bound drops station capacity and
keeps only the job chains. The proof of Proposition~\ref{prop:lb}
therefore applies unchanged to maximal lags, blocking stations, and buffer
capacities, which only remove feasible schedules. Maximal lags add a
difficulty that a one-step mask cannot remove because a deadline
$c_{ij}+\bar{\ell}_{ij}$ on the successor start can expire while every
compatible station is busy. A constructive policy can then reach a state
with no feasible completion. The benchmark contains neither maximal lags
nor blocking stations, so the experiments test only finite
buffers and uncertain lags. Uncertain lags are the one type that
changes the bound, because a bound computed from nominal lags estimates
the realized makespan without bounding it. One remedy trains on bounds
over sampled lag scenarios, as Smit \emph{et al.}~\cite{smit2025stochastic}
do for processing times. The uncertain-lag study evaluates a closed loop
instead, which writes each lag into the transition and the bound when
its operation completes (Section~\ref{sec:robust}).

\begin{table}[!t]
\caption{Changes to the Decision Process Under Neighboring Constraint
Types}
\label{tab:generality}
\centering
\footnotesize
\setlength{\tabcolsep}{3pt}
\begin{tabular}{@{}>{\raggedright\arraybackslash}p{0.19\columnwidth}>{\raggedright\arraybackslash}p{0.26\columnwidth}>{\raggedright\arraybackslash}p{0.21\columnwidth}>{\raggedright\arraybackslash}p{0.14\columnwidth}>{\raggedright\arraybackslash}p{0.10\columnwidth}@{}}
\toprule
Constraint & Transition & Action mask & Bound~\eqref{eq:lb} & Tested \\
\midrule
Minimal lag (FJSP-TL) & job ready at $c_{ij}+\ell_{ij}$, station free at $c_{ij}$ & unchanged & admissible & yes \\
Maximal lag & successor deadline $c_{ij}+\bar{\ell}_{ij}$ added & removes pairs that miss a deadline & admissible & no \\
Blocking station & station free at $c_{ij}+\ell_{ij}$ & unchanged & admissible & no \\
Buffer of capacity $B$ & start delayed until a slot is free & removes pairs that exceed $B$ & admissible & yes \\
Uncertain lag & revealed lag replaces the nominal lag & unchanged & estimate only & yes \\
\bottomrule
\end{tabular}
\end{table}

\section{Benchmark and Experimental Setup}\label{sec:benchmark}
\subsection{Benchmark and Test Sets}\label{sec:testsets}
No public FJSP benchmark contains post-operation time-lags, so this
study draws its orders from a generator built on the BCA
guidebook.\footnote{Generator, instance seeds, CP-SAT
references, trained models, and evaluation code:
\url{https://github.com/NTUZZH/FJSP-DRL-PPVC}.}
Table~\ref{tab:provenance} gives the source and the basis of every
parameter. The BCA PPVC guidebook~\cite{bca2017ppvc} fixes the operation
sequences, the \nStationTypes{} station types, and the
\nRoutingClasses{} routing classes, but it states no durations and no
station counts.

Building codes and coating data sheets anchor the three wait ranges,
and Section~S10 of the Supplementary Material quotes each limit. No
parameter of the generator is a factory measurement, so its processing
times and waits are planning values. Section~S7 of the Supplementary
Material bounds their effect on the ranking of SLAB and SPT at 10
modules.

The size ladder of 10, 20, 30, 50, and 80 modules on the
\FactoryDefault-station default factory spans the scale of a documented
development of about \ModulesPerFloor{} modules per
floor~\cite{unitedtec2022avenuesouth}. Ten modules form the training
size, 30 modules approximate one floor, and 50 and 80 modules form
production batches of two to three floors.

\begin{table}[!t]
\caption{Provenance of the Benchmark Parameters (Section Numbers Refer
to the BCA Guidebook~\cite{bca2017ppvc})}
\label{tab:provenance}
\centering
\footnotesize
\setlength{\tabcolsep}{3pt}
\begin{tabular}{@{}>{\raggedright\arraybackslash}p{0.27\columnwidth}>{\raggedright\arraybackslash}p{0.15\columnwidth}>{\raggedright\arraybackslash}p{0.25\columnwidth}>{\raggedright\arraybackslash}p{0.25\columnwidth}@{}}
\toprule
Parameter & Source & Value & Basis \\
\midrule
Operations per module & Guidebook & \nOpsPerModule{}; 18--22 with optional steps & Secs.~4.1--4.2 \\
Station types & Guidebook & \nStationTypes{}, from formwork to dispatch & Secs.~4.1--4.2 and 8 \\
Routing classes & Guidebook & Reinforced-concrete (RC) or steel shell, wet or dry fit-out & Secs.~4.1--4.2 \\
Wet-to-dry ratio per project & Guidebook & 3:1 & Sec.~2.4 \\
Order size & Project record & 10--80 modules & \ModulesPerFloor{} modules per floor~\cite{unitedtec2022avenuesouth} \\
Processing time & Estimate & 1--12\,h by operation & Authors' estimate \\
Curing wait & Estimate & \CureLo--\CureHi\,h & \cite{aci347,is456} \\
Ponding-test wait & Estimate & \PondLo--\PondHi\,h & \cite{ss637,astmd5957} \\
Paint drying wait per coat & Estimate & \DryLo--\DryHi\,h & \cite{nippon9000tds,jotunjotafloortds} \\
Omission rate of optional steps & Estimate & 25--50\% by step & No frequency in the guidebook \\
Duration spread across stations & Synthetic & $\pm20\%$ & Difficulty setting \\
Stations per factory & Synthetic & \FactoryDefault{} or \FactoryTight{} & Layouts are plant-specific (Sec.~2.1) \\
Shell mix across projects & Synthetic & RC and steel, 1:1 & One shell per real project (Appendix) \\
\bottomrule
\end{tabular}
\end{table}

Each size of the ladder has a fixed test set of \nInstTen{} orders at
10 modules, \nInstTwenty{} at 20, \nInstThirty{} at 30, and
\nInstFifty{} at each of 50 and 80 modules. Further sets change one
condition at a time, with a tight factory of \FactoryTight{}
stations, routes of 18 to 22 operations, or waits 1.5 times longer.
Independent confirmation sets of 100, 50, 50, and 50 orders at 10, 20,
50, and 80 modules come from generator seeds that no other set uses.
The seed ranges of the test, independent, and validation sets do not
overlap, so no test order enters training or model selection
(Section~S11 of the Supplementary Material).

\subsection{Training and Evaluation}\label{sec:training}
The SLAB policy runs on DANIEL without the optional aids. This network
has 28,834 parameters. PPO trains it for 2000 updates on orders
of 10 modules on the default factory, with every wait halved as in
Section~\ref{sec:trainorders} and the other hyperparameters of DANIEL at
their published values. The kept checkpoint minimizes the mean ratio of
validation makespan to fixed SPT references on halved-wait orders of 10
and 20 modules. Training is
repeated with three seeds. The control policies differ only in their
training orders or in the aids, and each also has three seeds. PLAN
(13,506 parameters) follows the same recipe, with one seed for each of
its control policies. HGNN
(204,250 parameters) keeps its published recipe of 1000 iterations with
one seed, and Section~S3 of the Supplementary Material gives its
checkpoint selection.

Each policy builds one schedule per test order by greedy decoding,
which takes the highest-scoring operation--station pair at every step.
Every evaluation uses the true waits. An independent validator, which
shares no code with the environment, checks every schedule of every
method against these waits. Makespans are means over the orders of a
test set, and over the three seeds where a policy has three. The
studies of uncertain waits, disruptions, buffers, routes, and
fine-tuning use the SLAB policy of the first seed.

\subsection{Comparators}\label{sec:comparators}
To our knowledge, DANIEL and HGNN are the only DRL schedulers for the
FJSP from IEEE Transactions journals whose code is public, and both are
backbones of this study. Besides these networks, eight dispatching rules
run in the SLAB transition, so every rule schedule respects the waits. Four of them rank candidates without
reading the waits, among them FIFO and SPT. The other four rules read
them, among them longest lag first (LLF), which prefers the operation
with the longest wait. Section~S1 defines all eight rules.

The metaheuristic comparator is the memetic algorithm of Caumond
\emph{et al.}~\cite{caumond2008memetic} for the job shop with
time-lags. This study extends it to flexible routes and seeds it with
the schedules of the eight rules. Its local search runs on every child,
because the published rate of 5\% gives longer schedules at every size
(Section~S1). The algorithm searches for 300\,s per order.

CP-SAT~\cite{perron2024cpsat} solves the model
\eqref{eq:lagprec}--\eqref{eq:capacity} in three modes. It starts cold,
from a hint that holds the complete SPT schedule, or, at 50 and 80
modules, from a hint that holds the SLAB schedule of the first seed. Each
mode runs with four workers for 300\,s per order at 10 to 30 modules
and for 600\,s at 50 and 80 modules. The cold and SPT-started modes also
run for 1800\,s on the first 20 orders of the 50- and 80-module sets.

\subsection{Uncertain Waits, Disruptions, and Buffers}\label{sec:protocols}
The uncertain-wait and disruption studies use the 100 orders of 10
modules and the first 20 orders of 50 modules. Section~S11 of the
Supplementary Material gives their full protocols. In the
uncertain-wait study, each realized wait equals the nominal wait times a
lognormal factor with mean 1 and a coefficient of variation (cv) of
0.1, 0.2, or 0.3. A wait becomes known only when its operation
completes. Every method plans with the nominal waits and meets the same
three realizations per order and cv. SLAB and SPT re-plan in a closed
loop that writes each revealed wait into the transition and the bound
before the next decision. Two plan-and-repair methods take a nominal plan from SLAB or from CP-SAT, given 60\,s at 10 modules and started from the SPT schedule with 600\,s on four cores at 50 modules, and shift it later to the realized waits. A rolling CP-SAT starts from the nominal
CP-SAT plan and re-solves the unstarted
operations for up to 5\,s whenever a revealed wait differs from its
nominal value. It runs only at 10 modules, because a 50-module order
reveals about 163 waits that would each trigger a re-solve. A single SLAB run with every realized wait known in
advance gives the perfect-information reference.

In the disruption study, one disruption strikes after 40\% of the
operations of the SLAB plan are committed. The six disruption types
double one wait, stop the busiest station, release an urgent module,
delay one module until its material arrives, lengthen an operation in
progress, or cancel the rest of a module. SLAB and the LLF rule
continue from the disrupted state and may re-route every uncommitted
operation. A genetic
algorithm (GA) re-plans the uncommitted operations for 5\,s, starting
from the right-shifted plan. Right-shift keeps every assignment and
sequence of the standing plan and moves starts later. A residual CP-SAT model keeps the committed operations and re-plans the rest for 5 or 30\,s with two workers, started from the right-shifted plan.

In the buffer study, one shared buffer of $B$ slots, such as a set of
curing bays, holds every module during its waits. CP-SAT models the buffer as a cumulative resource,
with 60\,s per order at 10 modules and 120\,s at 50 modules on the
first 20 orders. SLAB and SPT receive it as a mask that removes every
pair whose wait would exceed $B$, without retraining.

\subsection{Timing and Statistics}\label{sec:stats}
Training ran on one NVIDIA RTX PRO 5000 Blackwell GPU (48\,GB), and
the re-planning studies ran every method on the CPU. Every time-limited run of CP-SAT,
the memetic algorithm, or the GA ran on efficiency cores (4.7\,GHz) of an Intel
Core Ultra 9 285K that no other job used during the run. CP-SAT used
four workers on four cores in the comparison of Table~\ref{tab:main}
and for the 50-module nominal plans of the uncertain-wait study. Its
other runs in the re-planning and buffer studies used two workers on two
cores. Each
metaheuristic run used one core. The result of a time-limited search
depends on the computation it receives, so each run recorded its CPU
time divided by its wall time times its worker count. A run that used
its full time limit with a ratio below 0.9 was repeated, so every
reported run of this kind has a ratio of at least 0.9.

Comparisons are paired by order, and a policy with three seeds enters
with its per-order seed mean. A two-sided Wilcoxon signed-rank test
gives the $p$-value of each comparison. Holm's procedure corrects these
$p$-values within each family of comparisons against one rival, which
spans the five sizes, the three levels of wait uncertainty, or the six
disruption types. The matched-pairs rank-biserial correlation $r$ gives
the effect size, and $r<0$ means that SLAB gives the shorter schedules.
A 95\% percentile bootstrap interval accompanies the mean paired
difference of the two bounds.
Section~S2 of the Supplementary Material lists every test.

\section{Results}\label{sec:results}
\subsection{Necessity}\label{sec:necessity}
A learned dispatcher must see the waits when it acts. Any network
trained without waits can be executed in two ways. Its plan can be shifted
afterwards until every wait holds, or the network can dispatch in the
SLAB transition. The repaired plans are \RepairPenLo\% to \RepairPenHi\% longer (Table~\ref{tab:necessity}). Both executions use
the same trained weights, so the difference comes from the transition
alone. The penalty grows with order size on all three backbones,
because each repair delay passes on to the modules queued behind the
shifted one. The waits therefore belong in
the transition that executes the policy, whichever network scores the
pairs. Section~S9 of the Supplementary Material shows the two schedules
of one order as Gantt charts.

\begin{table}[!t]
\caption{Mean Makespan (h) of Networks Trained Without Waits Under Two
Executions}
\label{tab:necessity}
\centering
\footnotesize
\setlength{\tabcolsep}{3pt}
\begin{threeparttable}
\begin{tabular}{@{}llccccc@{}}
\toprule
 & & \multicolumn{5}{c}{Modules per order} \\
\cmidrule(l){3-7}
Network & Execution & 10 & 20 & 30 & 50 & 80 \\
\midrule
\multirow{3}{*}{DANIEL\tnote{a}}
 & Repaired plan & \RepairDanielTen & \RepairDanielTwenty & \RepairDanielThirty & \RepairDanielFifty & \RepairDanielEighty \\
 & SLAB dispatch & \DispatchDanielTen & \DispatchDanielTwenty & \DispatchDanielThirty & \DispatchDanielFifty & \DispatchDanielEighty \\
 & Increase (\%) & \RepairPenDanielTen & \RepairPenDanielTwenty & \RepairPenDanielThirty & \RepairPenDanielFifty & \RepairPenDanielEighty \\
\midrule
\multirow{3}{*}{PLAN}
 & Repaired plan & \RepairPlanTen & \RepairPlanTwenty & \RepairPlanThirty & \RepairPlanFifty & \RepairPlanEighty \\
 & SLAB dispatch & \DispatchPlanTen & \DispatchPlanTwenty & \DispatchPlanThirty & \DispatchPlanFifty & \DispatchPlanEighty \\
 & Increase (\%) & \RepairPenPlanTen & \RepairPenPlanTwenty & \RepairPenPlanThirty & \RepairPenPlanFifty & \RepairPenPlanEighty \\
\midrule
\multirow{3}{*}{HGNN}
 & Repaired plan & \RepairHgnnTen & \RepairHgnnTwenty & \RepairHgnnThirty & \RepairHgnnFifty & \RepairHgnnEighty \\
 & SLAB dispatch & \DispatchHgnnTen & \DispatchHgnnTwenty & \DispatchHgnnThirty & \DispatchHgnnFifty & \DispatchHgnnEighty \\
 & Increase (\%) & \RepairPenHgnnTen & \RepairPenHgnnTwenty & \RepairPenHgnnThirty & \RepairPenHgnnFifty & \RepairPenHgnnEighty \\
\bottomrule
\end{tabular}
\begin{tablenotes}\footnotesize
\item[a] Mean of three training seeds. PLAN and HGNN use one seed.
\end{tablenotes}
\end{threeparttable}
\end{table}

\subsection{Production Batches}\label{sec:perf}
One SLAB policy gives shorter schedules than every dispatching rule at
every size, and its lead grows with the order
(Table~\ref{tab:main}). Each of its three training seeds alone is
shorter than SPT at every size (Section~S1 of the Supplementary
Material). Their mean is \HstarVsSptGainTen\% shorter than SPT at
10 modules and \HstarVsSptGainEighty\% shorter at 80 modules. SPT is the strongest rule
from 30 modules on, and below that size it ties LLF, which is the
strongest of the four rules that read the waits. SLAB is
shorter than SPT on every order at 50 and 80 modules, and the paired
differences are significant at every size after Holm correction
(adjusted $p\le{}$\PvalSptHolmThirty, $r$ from $\RbisSptTen$ to
$\RbisSptEighty$). The lead also holds on the independent orders,
where SLAB is \HstarIndVsSptGainTen\% to \HstarIndVsSptGainEighty\%
shorter than SPT.

At production batches of 50 and 80 modules, SLAB also leads the search
methods that spend minutes on each order. It is \HstarVsMaGainFifty\%
and \HstarVsMaGainEighty\% shorter than the memetic algorithm with
300\,s per order and shorter on every order (adjusted
$p={}$\PvalMaHolmFifty{} at both sizes). SLAB is
\HstarVsMaLossTen\% longer than the memetic algorithm at 10 modules
(adjusted $p={}$\PvalMaHolmTen) and shorter from 20 modules on. Against
CP-SAT started from the SPT schedule
with 600\,s on four cores, it is \HstarVsCpWarmSptGainFifty\% and \HstarVsCpWarmSptGainEighty\% shorter (adjusted
$p={}$\PvalCpWarmHolmFifty{} at both sizes, $r=\RbisCpWarmFifty$).
Each seed alone also leads both methods at these sizes. More search time does not close
the gap (Fig.~\ref{fig:anytime}). With 1800\,s on the first 20 orders,
CP-SAT started from SPT reaches \CpLongWarmFifty{} and
\CpLongWarmEighty{}\,h, against \LongSlabFifty{} and
\LongSlabEighty{}\,h for SLAB on the same orders. Given 16 workers for
600\,s on these 80-module orders, it reaches \CpSixteenEighty\,h against
\LongSlabEighty{}\,h for SLAB (Section~S1). The SLAB
schedule is also the better starting point for a factory that can run a
solver. CP-SAT started from it shortens it by \CpWarmPolGainFifty{} and
\CpWarmPolGainEighty\% in 600\,s and reaches the shortest schedules of all methods at 50 and 80 modules. CP-SAT on the flow-line restriction of the same orders, which fixes one station per module and type, ends at \FlowCpLineFifty{} and \FlowCpLineEighty{}\,h (Section~S6), still above SLAB.

CP-SAT remains the better choice for small orders. It proves
\CpOptTen{} of the \nInstTen{} schedules of 10 modules optimal within
300\,s, and it reaches the SLAB makespan within 5\,s on every order,
after a median of \CpReachMedColdTen\,s from a cold start. SLAB is
\HstarVsCpWarmSptTen\% longer than the CP-SAT mean (adjusted
$p={}$\PvalCpWarmHolmTen, $r=\RbisCpWarmTen$), and at 20 modules the
difference falls to \HstarVsCpWarmSptTwenty\%. From 30 modules on, SLAB
is shorter than every CP-SAT mode that does not start from its own
schedule, and at 50 and 80 modules no such run reaches the SLAB makespan
within 600\,s or 1800\,s (Section~S1).
This crossover places the operating range of SLAB at batches
of one floor and more.

\begin{table}[!t]
\caption{Mean Makespan (h) on the Default Factory. Bold Marks the
Shortest Mean in Each Column of Each Block.}
\label{tab:main}
\centering
\footnotesize
\setlength{\tabcolsep}{3.5pt}
\begin{threeparttable}
\begin{tabular}{@{}lccccc@{}}
\toprule
 & \multicolumn{5}{c}{Modules per order} \\
\cmidrule(l){2-6}
Method & 10 & 20 & 30 & 50 & 80 \\
\midrule
Orders & \nInstTen & \nInstTwenty & \nInstThirty & \nInstFifty & \nInstEighty \\
\midrule
SPT rule & \SptTen & \SptTwenty & \SptThirty & \SptFifty & \SptEighty \\
LLF rule & \LlfTen & \LlfTwenty & \LlfThirty & \LlfFifty & \LlfEighty \\
Memetic, 300\,s & \MaTen & \MaTwenty & \MaThirty & \MaFifty & \MaEighty \\
CP-SAT, cold\tnote{a} & \textbf{\CpColdTen} & \CpColdTwenty & \CpColdThirty & \CpColdFifty & \CpColdEighty \\
CP-SAT, from SPT\tnote{a} & \textbf{\CpWarmSptTen} & \textbf{\CpWarmSptTwenty} & \CpWarmSptThirty & \CpWarmSptFifty & \CpWarmSptEighty \\
CP-SAT, from SLAB\tnote{a,b} & -- & -- & -- & \textbf{\CpWarmPolFifty} & \textbf{\CpWarmPolEighty} \\
SLAB\tnote{c} & \HstarTen & \HstarTwenty & \textbf{\HstarThirty} & \HstarFifty & \HstarEighty \\
\quad Std over seeds & \HstarStdTen & \HstarStdTwenty & \HstarStdThirty & \HstarStdFifty & \HstarStdEighty \\
\midrule
\multicolumn{6}{@{}l}{\emph{Independent test sets}} \\
SPT rule & \SptIndTen & \SptIndTwenty & -- & \SptIndFifty & \SptIndEighty \\
SLAB\tnote{c} & \textbf{\HstarIndTen} & \textbf{\HstarIndTwenty} & -- & \textbf{\HstarIndFifty} & \textbf{\HstarIndEighty} \\
\bottomrule
\end{tabular}
\begin{tablenotes}\footnotesize
\item[a] Four workers, 300\,s at 10 to 30 modules and 600\,s at 50 and
80 modules.
\item[b] Started from the schedule of the first SLAB seed built on the CPU (\CpWarmPolHintFifty{} and \CpWarmPolHintEighty{}\,h, Section~S5).
\item[c] Mean of three training seeds.
\end{tablenotes}
\end{threeparttable}
\end{table}

\begin{figure}[!t]
  \centering
  \includegraphics[width=\columnwidth]{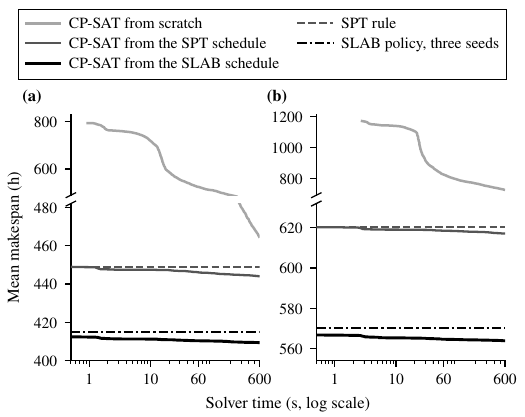}
  \caption{Mean of the best CP-SAT makespan found by each solver time,
  over the orders of (a) 50 and (b) 80 modules. Horizontal lines mark the SPT
  rule and the SLAB policy. The broken axis shows the start of the runs
  from scratch.}
  \label{fig:anytime}
\end{figure}

\subsection{Training Orders}\label{sec:mechanism}
\begin{table*}[!t]
\caption{Mean Makespan (h) of Policies Trained on Different Orders and
the Share $\sigma$ of Idle Time Spent Queuing in Those Orders. Bold Marks
the Shortest Mean in Each Column.}
\label{tab:trainorders}
\centering
\footnotesize
\setlength{\tabcolsep}{4pt}
\begin{threeparttable}
\begin{tabular}{@{}llcccccccc@{}}
\toprule
 & & & \multicolumn{5}{c}{Default factory, modules per order} & & \FactoryTight{} stations \\
\cmidrule(lr){4-8}\cmidrule(l){10-10}
Network & Training orders & $\sigma$ & 10 & 20 & 30 & 50 & 80 & & 10 \\
\midrule
\multirow{5}{*}{DANIEL\tnote{a}}
 & Without waits & \ContTenLagZero & \TrainLagFreeTen$\,\pm\,$\TrainLagFreeStdTen & \TrainLagFreeTwenty$\,\pm\,$\TrainLagFreeStdTwenty & \TrainLagFreeThirty$\,\pm\,$\TrainLagFreeStdThirty & \TrainLagFreeFifty$\,\pm\,$\TrainLagFreeStdFifty & \TrainLagFreeEighty$\,\pm\,$\TrainLagFreeStdEighty & & \TrainLagFreeTight$\,\pm\,$\TrainLagFreeStdTight \\
 & Halved waits (SLAB) & \ContTenLagHalf & \TrainHalfTen$\,\pm\,$\TrainHalfStdTen & \textbf{\TrainHalfTwenty}$\,\pm\,$\TrainHalfStdTwenty & \textbf{\TrainHalfThirty}$\,\pm\,$\TrainHalfStdThirty & \textbf{\TrainHalfFifty}$\,\pm\,$\TrainHalfStdFifty & \textbf{\TrainHalfEighty}$\,\pm\,$\TrainHalfStdEighty & & \TrainHalfTight$\,\pm\,$\TrainHalfStdTight \\
 & Full waits & \ContTenLagFull & \TrainFullTen$\,\pm\,$\TrainFullStdTen & \TrainFullTwenty$\,\pm\,$\TrainFullStdTwenty & \TrainFullThirty$\,\pm\,$\TrainFullStdThirty & \TrainFullFifty$\,\pm\,$\TrainFullStdFifty & \TrainFullEighty$\,\pm\,$\TrainFullStdEighty & & \TrainFullTight$\,\pm\,$\TrainFullStdTight \\
 & Full waits, aids & \ContTenLagFull & \TrainFullAidsTen$\,\pm\,$\TrainFullAidsStdTen & \TrainFullAidsTwenty$\,\pm\,$\TrainFullAidsStdTwenty & \TrainFullAidsThirty$\,\pm\,$\TrainFullAidsStdThirty & \TrainFullAidsFifty$\,\pm\,$\TrainFullAidsStdFifty & \TrainFullAidsEighty$\,\pm\,$\TrainFullAidsStdEighty & & \TrainFullAidsTight$\,\pm\,$\TrainFullAidsStdTight \\
 & Full waits, aids, \FactoryTight{} stations & \ContTight & \textbf{\TrainTightTen}$\,\pm\,$\TrainTightStdTen & \textbf{\TrainTightTwenty}$\,\pm\,$\TrainTightStdTwenty & \TrainTightThirty$\,\pm\,$\TrainTightStdThirty & \TrainTightFifty$\,\pm\,$\TrainTightStdFifty & \TrainTightEighty$\,\pm\,$\TrainTightStdEighty & & \textbf{\TrainTightTight}$\,\pm\,$\TrainTightStdTight \\
\midrule
\multirow{4}{*}{PLAN\tnote{b}}
 & Without waits & \ContTenLagZero & \PlanLagFreeTen & \PlanLagFreeTwenty & \PlanLagFreeThirty & \PlanLagFreeFifty & \PlanLagFreeEighty & & \PlanLagFreeTight \\
 & Halved waits & \ContTenLagHalf & \PlanHalfTen$\,\pm\,$\PlanHalfStdTen & \PlanHalfTwenty$\,\pm\,$\PlanHalfStdTwenty & \PlanHalfThirty$\,\pm\,$\PlanHalfStdThirty & \PlanHalfFifty$\,\pm\,$\PlanHalfStdFifty & \PlanHalfEighty$\,\pm\,$\PlanHalfStdEighty & & \PlanHalfTight$\,\pm\,$\PlanHalfStdTight \\
 & Full waits & \ContTenLagFull & \PlanFullTen & \PlanFullTwenty & \PlanFullThirty & \PlanFullFifty & \PlanFullEighty & & \PlanFullTight \\
 & Full waits, aids & \ContTenLagFull & \PlanFullAidsTen & \PlanFullAidsTwenty & \PlanFullAidsThirty & \PlanFullAidsFifty & \PlanFullAidsEighty & & \PlanFullAidsTight \\
\midrule
\multicolumn{3}{@{}l}{$\sigma$ of the test orders at the true waits} & \ContTenLagFull & \ContTwenty & \ContThirty & \ContFifty & \ContEighty & & \ContTight \\
\bottomrule
\end{tabular}
\begin{tablenotes}\footnotesize
\item Values are mean $\pm$ standard deviation over three training
seeds where a $\pm$ is shown.
\item[a] Network without the aids unless the row names them.
\item[b] One seed per row except the halved waits.
\end{tablenotes}
\end{threeparttable}
\end{table*}
The policy trained with halved waits gives the shortest schedules, or
schedules within one seed standard deviation of the shortest, among the
training options at every size of the default factory
(Table~\ref{tab:trainorders}). The same network trained at the
full waits matches SLAB at 10 modules. At 50 and 80 modules it reaches
\TrainFullFifty{} and \TrainFullEighty{}\,h against \TrainHalfFifty{}
and \TrainHalfEighty{}\,h. The network trained without waits ties SLAB
at 50 and 80 modules but is longer where the waits fill most of the
idle time. The factor 0.5 and its adoption condition were fixed before the first training run with it, and the condition required the policy to be best at both ends of the size ladder. Two further seeds and the
independent sets of Table~\ref{tab:main}, which took no part in this
choice, confirmed the result. Single-seed runs at the factors 0.75 and
0.25 ($\sigma=\ContTenLagThreeQuarter$ and
\ContTenLagQuarter) reach \TrainThreeQuarterEighty{} and
\TrainQuarterEighty{}\,h at 80 modules (Section~S1.4 of the Supplementary Material). The schedules of production batches thus shorten as the training share of queuing rises to one half.

Both losses follow the share of idle time spent queuing. Raising the share through scarce stations also helps. The
\FactoryTight-station factory raises $\sigma$ of a 10-module order to
\ContTight{} at the full waits. Training there with the full waits and
aids shortens the 50- and 80-module schedules to \TrainTightFifty{} and
\TrainTightEighty{}\,h, from \TrainFullAidsFifty{} and
\TrainFullAidsEighty{}\,h for the same recipe on the default factory.

The recipe carries over to PLAN. With halved waits, PLAN is best or
tied-best among its training options at 10, 20, 30, and 80 modules and
second to wait-free training at 50 modules. At 80 modules it reaches
\PlanHalfEighty{}\,h, against \PlanFullEighty{}\,h with full waits.
HGNN is the exception. None of its training options under its published
recipe is shorter than SPT beyond 10 modules. Section~S3 of the
Supplementary Material lists every HGNN option. The necessity result of
Table~\ref{tab:necessity} still holds on HGNN.

On the default factory, neither the tighter bound nor the input aids
shorten the learned schedules. Over the decision states of the SLAB schedules, the
lag-augmented bound averages \BoundTightAvgAugTen{} to
\BoundTightAvgAugEighty{} of the final makespan at 10 to 80 modules,
against \BoundTightAvgFreeTen{} to \BoundTightAvgFreeEighty{} for the
lag-free bound. The bound comparison keeps the network and the full
waits fixed and swaps the bound both in the reward and in the
completion-time estimates of the state. Its mean paired difference
(lag-free minus lag-augmented bound) is $\BoundDiffTen$,
$\BoundDiffFifty$, and $\BoundDiffEighty$\,h at 10, 50, and 80 modules,
and all three 95\% intervals contain zero. At 20 and 30 modules, the
difference is
$\BoundDiffTwenty$ and $\BoundDiffThirty$\,h, with intervals of
$[\BoundCiLoTwenty,\BoundCiHiTwenty]$ and
$[\BoundCiLoThirty,\BoundCiHiThirty]$\,h, so the lag-free bound gives
slightly shorter schedules there. Under the lag-augmented bound, the three training seeds also
spread further apart from 20 modules on. Adding the input aids leaves
the full-wait schedules unchanged or longer at every size of the default
factory. Section~S4 reports both comparisons per seed and the
tightness of both bounds.

\subsection{Uncertain Waits and Disruptions}\label{sec:robust}
Re-planning as the waits become known keeps SLAB within
\StochGapPiMax\% of the makespan it reaches with every wait known in
advance, at every cv and at both sizes
(Table~\ref{tab:uncertain}). Keeping the nominal plan and shifting it to
the realized waits costs much more. The repaired SLAB plan ends at
\StochRepairTenCvThree{} and \StochRepairFiftyCvThree{}\,h at 10 and
50 modules for a cv of 0.3, against \StochReactTenCvThree{} and
\StochReactFiftyCvThree{}\,h with re-planning. In the 10-module study,
the rolling CP-SAT is shorter than SLAB at every cv, at the cost of a
60-s nominal plan and \StochCpRollReplanS{}\,s of re-solving per order.

SLAB gives the lowest mean makespan at 50 modules after five of the six
disruption types (Table~\ref{tab:disrupt}). The LLF and SPT rules
re-planning in the same loop end above SLAB on every order, the better
of the two by \RuleGapMinFifty{} to \RuleGapMaxFifty\,h on average.
Residual CP-SAT ends above SLAB for the same five types, by up to
\ResFiveGapMaxFifty\% given 5\,s and \ResGapMaxFifty\% given 30\,s, and
with 30\,s the difference is significant for three of them
(Section~S2). After a breakdown,
which stops the busiest station for half a mean processing time,
right-shift repair keeps the standing plan at \DisBrkRsFifty{}\,h
against \DisBrkSlabFifty{}\,h for SLAB, and the GA and residual CP-SAT
also end below SLAB. With an outage of 8\,h or 24\,h, SLAB gives the
shortest re-plan, by \SevMarginEightFifty{} and
\SevMarginTwentyFourFifty\,h over residual CP-SAT given 30\,s
(Section~S8). The ranking reverses at 10 modules, where residual CP-SAT
given 5\,s proves at least \ResOptMinTen{} of 100 re-plans optimal for
every type and ends \ResGainMinTen{} to \ResGainMaxTen\,h below SLAB,
and it stays shorter after the longer outages. SLAB rebuilds much of
the remaining plan, and at 50 modules it makes \DisNervSlabFifty{} changes per disruption on average, against
\DisNervGaFifty{} for the GA and \DisNervResThirtyFifty{} for residual
CP-SAT given 30\,s. A penalty on each change in the residual model keeps
its re-plans of breakdowns and doubled waits optimal within
\StabTimeMaxTen\,s at 10 modules, and these plans match the makespan of
SLAB with fewer changes. At 50 modules and a penalty of 0.5\,h per
change, residual CP-SAT given 30\,s stays within \StabInflBrkFifty\% of
the standing plan after a short breakdown with \StabNervBrkFifty{}
changes, against \SlabInflBrkFifty\% and \SlabNervBrkFifty{} changes for
SLAB. After a doubled wait, it makes \StabNervLagFifty{} changes but ends
\StabInflLagFifty\% above the standing plan, against
\SlabInflLagFifty\% for SLAB (Section~S8).

\begin{table}[!t]
\caption{Mean Realized Makespan (h) Under Uncertain Waits, by
Coefficient of Variation of the Waits. Bold Marks the Shortest Mean in
Each Column Above the Reference Row.}
\label{tab:uncertain}
\centering
\footnotesize
\setlength{\tabcolsep}{3pt}
\begin{tabular}{@{}lcccccc@{}}
\toprule
 & \multicolumn{3}{c}{10 modules} & \multicolumn{3}{c}{50 modules} \\
\cmidrule(lr){2-4}\cmidrule(l){5-7}
Method & 0.1 & 0.2 & 0.3 & 0.1 & 0.2 & 0.3 \\
\midrule
SLAB, re-plan & \StochReactTenCvOne & \StochReactTenCvTwo & \StochReactTenCvThree & \textbf{\StochReactFiftyCvOne} & \textbf{\StochReactFiftyCvTwo} & \textbf{\StochReactFiftyCvThree} \\
SLAB, repair & \StochRepairTenCvOne & \StochRepairTenCvTwo & \StochRepairTenCvThree & \StochRepairFiftyCvOne & \StochRepairFiftyCvTwo & \StochRepairFiftyCvThree \\
SPT, re-plan & \StochSptTenCvOne & \StochSptTenCvTwo & \StochSptTenCvThree & \StochSptFiftyCvOne & \StochSptFiftyCvTwo & \StochSptFiftyCvThree \\
CP-SAT, repair & \StochCpRepairTenCvOne & \StochCpRepairTenCvTwo & \StochCpRepairTenCvThree & \StochCpRepairFiftyCvOne & \StochCpRepairFiftyCvTwo & \StochCpRepairFiftyCvThree \\
CP-SAT, rolling & \textbf{\StochCpRollTenCvOne} & \textbf{\StochCpRollTenCvTwo} & \textbf{\StochCpRollTenCvThree} & -- & -- & -- \\
\midrule
SLAB, waits known & \StochPerfInfoTenCvOne & \StochPerfInfoTenCvTwo & \StochPerfInfoTenCvThree & \StochPerfInfoFiftyCvOne & \StochPerfInfoFiftyCvTwo & \StochPerfInfoFiftyCvThree \\
\bottomrule
\end{tabular}
\end{table}

\begin{table}[!t]
\caption{Mean Makespan (h) After One Disruption. Bold Marks the
Smallest Value in Each Row Across All Methods.}
\label{tab:disrupt}
\centering
\footnotesize
\setlength{\tabcolsep}{2.6pt}
\begin{threeparttable}
\begin{tabular}{@{}lcccccc@{}}
\toprule
 & & & & & \multicolumn{2}{c}{CP-SAT\tnote{a}} \\
\cmidrule(l){6-7}
Disruption & SLAB & LLF & GA & Shift & 5\,s & 30\,s \\
\midrule
\multicolumn{7}{@{}l}{\emph{10 modules, 100 orders}} \\
Wait doubled & \DisLagSlabTen & \DisLagLlfTen & \DisLagGaTen & \DisLagRsTen & \textbf{\DisLagResFiveTen} & \textbf{\DisLagResThirtyTen} \\
Breakdown & \DisBrkSlabTen & \DisBrkLlfTen & \DisBrkGaTen & \DisBrkRsTen & \textbf{\DisBrkResFiveTen} & \textbf{\DisBrkResThirtyTen} \\
Urgent module & \DisUrgSlabTen & \DisUrgLlfTen & \DisUrgGaTen & \DisUrgRsTen & \textbf{\DisUrgResFiveTen} & \textbf{\DisUrgResThirtyTen} \\
Material delay & \DisMatSlabTen & \DisMatLlfTen & \DisMatGaTen & \DisMatRsTen & \textbf{\DisMatResFiveTen} & \textbf{\DisMatResThirtyTen} \\
Longer operation & \DisPtdSlabTen & \DisPtdLlfTen & \DisPtdGaTen & \DisPtdRsTen & \textbf{\DisPtdResFiveTen} & \textbf{\DisPtdResThirtyTen} \\
Cancellation & \DisCanSlabTen & \DisCanLlfTen & \DisCanGaTen & \DisCanRsTen & \textbf{\DisCanResFiveTen} & \textbf{\DisCanResThirtyTen} \\
Changes\tnote{b} & \DisNervSlabTen & \DisNervLlfTen & \DisNervGaTen & \textbf{\DisNervRsTen} & \DisNervResFiveTen & \DisNervResThirtyTen \\
Time (s) & \DisTimeSlabTen & \DisTimeLlfTen & \DisTimeGaTen & \textbf{\DisTimeRsTen} & \DisTimeResFiveTen & \DisTimeResThirtyTen \\
\midrule
\multicolumn{7}{@{}l}{\emph{50 modules, 20 orders}} \\
Wait doubled & \textbf{\DisLagSlabFifty} & \DisLagLlfFifty & \DisLagGaFifty & \DisLagRsFifty & \DisLagResFiveFifty & \DisLagResThirtyFifty \\
Breakdown & \DisBrkSlabFifty & \DisBrkLlfFifty & \DisBrkGaFifty & \DisBrkRsFifty & \DisBrkResFiveFifty & \textbf{\DisBrkResThirtyFifty} \\
Urgent module & \textbf{\DisUrgSlabFifty} & \DisUrgLlfFifty & \DisUrgGaFifty & \DisUrgRsFifty & \DisUrgResFiveFifty & \DisUrgResThirtyFifty \\
Material delay & \textbf{\DisMatSlabFifty} & \DisMatLlfFifty & \DisMatGaFifty & \DisMatRsFifty & \DisMatResFiveFifty & \DisMatResThirtyFifty \\
Longer operation & \textbf{\DisPtdSlabFifty} & \DisPtdLlfFifty & \DisPtdGaFifty & \DisPtdRsFifty & \DisPtdResFiveFifty & \DisPtdResThirtyFifty \\
Cancellation & \textbf{\DisCanSlabFifty} & \DisCanLlfFifty & \DisCanGaFifty & \DisCanRsFifty & \DisCanResFiveFifty & \DisCanResThirtyFifty \\
Changes\tnote{b} & \DisNervSlabFifty & \DisNervLlfFifty & \DisNervGaFifty & \DisNervRsFifty & \textbf{\DisNervResFiveFifty} & \DisNervResThirtyFifty \\
Time (s) & \DisTimeSlabFifty & \DisTimeLlfFifty & \DisTimeGaFifty & \textbf{\DisTimeRsFifty} & \DisTimeResFiveFifty & \DisTimeResThirtyFifty \\
\bottomrule
\end{tabular}
\begin{tablenotes}\footnotesize
\item LLF re-plans in the same closed loop as SLAB. The GA searches for
5\,s, and Shift is the right-shift of the standing plan.
\item[a] Residual model that keeps the committed operations, two workers.
\item[b] Mean over the six types of the station reassignments plus the
starts moved by more than 1\,h. An operation with both counts twice.
\end{tablenotes}
\end{threeparttable}
\end{table}

\subsection{Deployment Limits}\label{sec:deploy}
A finite wait buffer leaves the makespan of SLAB nearly unchanged until its capacity falls below about three quarters of the peak number of modules that wait at once in unconstrained schedules (Table~\ref{tab:deploy}). At 50 modules
this free peak has a median of \BufPeakFreeFifty{} modules. The
masked SLAB policy averages \BufHstarFiftySixteen{}\,h with $B=16$,
against \BufHstarFiftyFree{}\,h without a limit. It rises to
\BufHstarFiftyTwelve{}\,h with $B=12$, close to masked SPT, and CP-SAT
with 120\,s is shorter from this capacity down. At 10 modules, CP-SAT is
shorter at every capacity, and masked SPT overtakes SLAB at $B\le3$,
which is one third of the free peak of \BufPeakFreeTen{} modules.

SLAB transfers to variable routes and to longer waits without
retraining, and under both changes it gives shorter schedules than SPT
at every tested size. A policy trained on these routes at the full waits is shorter at 10
modules and longer from 20 modules on. On orders of projects that use
one shell type, reinforced concrete or steel, SLAB without retraining is
\ShellGainRcFifty\% to \ShellGainSteelTen\% shorter than SPT at 10 and
50 modules (Section~S1 of the Supplementary Material). With every wait 1.5
times longer, fine-tuning for 400 updates at half the new waits leaves
the makespan of SLAB nearly unchanged, so longer waits need no
retraining. Fine-tuning shortens schedules on a new station layout. On the
\FactoryTight-station factory, SLAB without retraining averages
\FtTightZs{}\,h against \MtightSptMean{}\,h for SPT. Fine-tuning for
400 updates, which takes \FtGpuMinutes{} GPU-minutes, brings it to
\FtTightAfter{}\,h, close to the \FtTightScratch{}\,h of policies
trained on that factory from the start.

Without a solver or a GPU, SLAB schedules an order of 10, 50, and 80
modules on one CPU core in \LatSlabTen{}, \LatSlabFifty{}, and
\LatSlabEighty{}\,s, with a peak memory of \LatSlabMemMb{}\,MB.
Training it once takes \TrainGpuHours{} GPU-hours.

\begin{table}[!t]
\caption{Mean Makespan (h) at the Deployment Limits. Bold Marks the
Shortest Mean in Each Row.}
\label{tab:deploy}
\centering
\footnotesize
\setlength{\tabcolsep}{4pt}
\begin{threeparttable}
\begin{tabular}{@{}llccc@{}}
\toprule
Condition & Modules & SLAB\tnote{a} & SPT & Reference \\
\midrule
\multicolumn{5}{@{}l}{\emph{Shared wait buffer of $B$ modules, reference CP-SAT}} \\
No limit & 10 & \BufHstarTenFree & \BufSptTenFree & \textbf{\BufCpTenFree} \\
$B=6$ & 10 & \BufHstarTenSix & \BufSptTenSix & \textbf{\BufCpTenSix} \\
$B=4$ & 10 & \BufHstarTenFour & \BufSptTenFour & \textbf{\BufCpTenFour} \\
$B=3$ & 10 & \BufHstarTenThree & \BufSptTenThree & \textbf{\BufCpTenThree} \\
No limit & 50 & \textbf{\BufHstarFiftyFree} & \BufSptFiftyFree & \BufCpFiftyFree \\
$B=16$ & 50 & \textbf{\BufHstarFiftySixteen} & \BufSptFiftySixteen & \BufCpFiftySixteen \\
$B=12$ & 50 & \BufHstarFiftyTwelve & \BufSptFiftyTwelve & \textbf{\BufCpFiftyTwelve} \\
$B=8$ & 50 & \BufHstarFiftyEight & \BufSptFiftyEight & \textbf{\BufCpFiftyEight} \\
$B=4$ & 50 & \BufHstarFiftyFour & \BufSptFiftyFour & \textbf{\BufCpFiftyFour} \\
\midrule
\multicolumn{5}{@{}l}{\emph{Routes of 18 to 22 operations, reference trained on them\tnote{b}}} \\
Variable routes & 10 & \VarHstarZsTen & \VarSptTen & \textbf{\VarTrainedTen} \\
Variable routes & 20 & \textbf{\VarHstarZsTwenty} & \VarSptTwenty & \VarTrainedTwenty \\
Variable routes & 50 & \textbf{\VarHstarZsFifty} & \VarSptFifty & \VarTrainedFifty \\
Variable routes & 80 & \textbf{\VarHstarZsEighty} & \VarSptEighty & \VarTrainedEighty \\
\midrule
\multicolumn{5}{@{}l}{\emph{New conditions, reference SLAB fine-tuned for 400 updates}} \\
\FactoryTight{} stations & 10 & \FtTightZs & \MtightSptMean & \textbf{\FtTightAfter} \\
Waits $\times1.5$ & 10 & \textbf{\FtLagZsTen} & \FtLagSptTen & \FtLagAfterTen \\
Waits $\times1.5$ & 50 & \textbf{\FtLagZsFifty} & \FtLagSptFifty & \textbf{\FtLagAfterFifty} \\
\bottomrule
\end{tabular}
\begin{tablenotes}\footnotesize
\item[a] Policy of the first training seed, run on the CPU without retraining.
\item[b] Trained at the full waits with both input aids.
\end{tablenotes}
\end{threeparttable}
\end{table}

\section{Discussion}\label{sec:discussion}
\textbf{The order size decides which scheduler a factory should run.}
The memetic algorithm builds shorter schedules than the SLAB policy on
orders of \BatchSmall{} modules, and CP-SAT given
\AnytimeBudgetMid{}\,s on orders of \BatchSmall{} and 20 modules
(Table~\ref{tab:main}). At \BatchSmall{} modules, CP-SAT also reaches the
makespan of the policy within 5\,s on every order.
From 30 modules on, the policy outperforms every comparator that does
not start from its schedule, so a factory without a solver can run it
alone at these sizes. A factory that can run CP-SAT for minutes per
order can instead start CP-SAT from the policy's schedule.

\textbf{The order-size boundary also applies to re-planning and finite buffers.} At
\BatchSmall{} modules, rolling CP-SAT ends below the policy at every level
of wait uncertainty, and residual CP-SAT given 5\,s after every
disruption type (Tables~\ref{tab:uncertain} and~\ref{tab:disrupt}). At \BatchProdLo{} modules,
where rolling CP-SAT was not run, a CP-SAT plan built once from the SPT
schedule with 600\,s and repaired to the realized waits ends at least
\CpGapMinFifty\% above the closed-loop re-planning of the policy
(Table~\ref{tab:uncertain}). When CP-SAT re-plans the residual from the
right-shifted plan for 30\,s, it comes within \ResGapMaxFifty\% of the
policy and is shorter only after a short breakdown
(Table~\ref{tab:disrupt}). Finite buffers set a second boundary. At \BatchProdLo{} modules, CP-SAT with the buffer as a constraint gives shorter schedules once the capacity falls below about three quarters of the free peak (Table~\ref{tab:deploy}).

\textbf{Several limits qualify these results.} The policy minimizes the
makespan only, and Section~S5 of the Supplementary Material reports four
further measures of its schedules. Each further objective would need a
reward term with its own bound, because Proposition~\ref{prop:lb} covers
the makespan alone. The policy has no such term for plan stability, so
a re-plan at \BatchProdLo{} modules makes \DisNervSlabFifty{} changes on
average. Processing times and waits are engineering estimates, and
buffer capacities are unknown. The results therefore compare methods on
one benchmark and need factory data before they predict absolute hours. Each wait is a minimal lag that does not
depend on the station, and a maximal lag, for example a deadline for the
next paint coat, lies outside the formulation. The tests stay within
the \nRoutingClasses{} routing classes and \nStationTypes{} station
types of the guidebook, so route templates of other products and new
station types remain untested. Production-batch results
hold for DANIEL and PLAN, which share one state representation, while
HGNN stays above the SPT rule beyond its training size under every
training option. The factor 0.5 was chosen in one test against full and
zero waits, and a single-seed run at 0.25 gives shorter schedules than
the same seed at 0.5 from 30 modules on and longer ones at 10 and 20
modules.

\textbf{The lag term of the bound keeps each reward free of waits that
no assignment can avoid, but it does not improve the policy.} With the
lag-free bound, which is looser and also admissible, the same network
builds schedules that are not distinguishable at \BatchSmall{},
\BatchProdLo{}, and \BatchProdHi{} modules and slightly shorter at 20
and 30 modules (Section~S4.1). The quality of the policy thus comes from the transition in
which it acts and from the orders on which it trains.

\section{Conclusion}\label{sec:conclusion}
This paper extended constructive DRL scheduling from the flexible job
shop to the flexible job shop with the job-blocking, machine-free
time-lags of a PPVC module factory. SLAB places each wait in the
transition that trains and executes a dispatcher and in the lower bound
whose decrease is the reward. Proposition~\ref{prop:lb} proves that this
bound stays admissible, so the return remains the makespan objective.
The transition is necessary, because the same network that plans without
it and is then repaired builds schedules \RepairPenLo\% to
\RepairPenHi\% longer on three backbones. Halving the waits of \BatchSmall-module training
orders gives queuing the share of idle time that it has in a 20-module
order. The policy therefore learns station contention as well as the
waits. One policy trained on these orders builds the shortest schedules
among the methods that need no solver from 20 to \BatchProdHi{} modules.
At production batches of \BatchProdLo{} and \BatchProdHi{} modules, it
also outperforms a time-lag memetic algorithm and CP-SAT started from
SPT and given \AnytimeBudget{}\,s. The memetic algorithm remains shorter
on orders of \BatchSmall{} modules, and CP-SAT on orders of
\BatchSmall{} and 20 modules. A module factory can therefore train one
dispatcher on small orders with shortened waits and run it without a
solver on production batches of \BatchProdLo{} to \BatchProdHi{}
modules.

\FloatBarrier
\balance
\bibliographystyle{IEEEtran}
\bibliography{refs}

\begin{thebibliography}{10}
\providecommand{\url}[1]{#1}
\csname url@samestyle\endcsname
\providecommand{\newblock}{\relax}
\providecommand{\bibinfo}[2]{#2}
\providecommand{\BIBentrySTDinterwordspacing}{\spaceskip=0pt\relax}
\providecommand{\BIBentryALTinterwordstretchfactor}{4}
\providecommand{\BIBentryALTinterwordspacing}{\spaceskip=\fontdimen2\font plus
\BIBentryALTinterwordstretchfactor\fontdimen3\font minus \fontdimen4\font\relax}
\providecommand{\BIBforeignlanguage}[2]{{%
\expandafter\ifx\csname l@#1\endcsname\relax
\typeout{** WARNING: IEEEtran.bst: No hyphenation pattern has been}%
\typeout{** loaded for the language `#1'. Using the pattern for}%
\typeout{** the default language instead.}%
\else
\language=\csname l@#1\endcsname
\fi
#2}}
\providecommand{\BIBdecl}{\relax}
\BIBdecl

\bibitem{hammad2020novel}
A.~W.~A. Hammad, H.~Grzybowska, M.~Sutrisna, A.~Akbarnezhad, and A.~Haddad, ``A novel mathematical optimisation model for the scheduling of activities in modular construction factories,'' \emph{Constr. Manage. Econ.}, vol.~38, no.~6, pp. 534--551, 2020.

\bibitem{peiris2023production}
A.~Peiris, F.~K.~P. Hui, C.~Duffield, and T.~Ngo, ``Production scheduling in modular construction: Metaheuristics and future directions,'' \emph{Autom. Constr.}, vol. 150, 2023, {Art. no. 104851}.

\bibitem{bca2017ppvc}
\BIBentryALTinterwordspacing
{Building and Construction Authority}, ``Design for manufacturing and assembly ({DfMA}): Prefabricated prefinished volumetric construction ({PPVC}) guidebook,'' 2017, {Singapore. Accessed: Jun. 13, 2026}. [Online]. Available: \url{https://www1.bca.gov.sg/growth-and-transformation/productivity/design-for-manufacturing-and-assembly-dfma/prefabricated-prefinished-volumetric-construction-ppvc/}
\BIBentrySTDinterwordspacing

\bibitem{unitedtec2022avenuesouth}
\BIBentryALTinterwordspacing
{United Tec Construction}, ``{Avenue South Residence}: World's tallest modular building project,'' Main contractor project record, 2022, {Singapore. Accessed: Sep. 10, 2026}. [Online]. Available: \url{https://www.unitedtec.com.sg/projects/avenue-south-residence/}
\BIBentrySTDinterwordspacing

\bibitem{perron2024cpsat}
\BIBentryALTinterwordspacing
L.~Perron and F.~Didier, ``{CP-SAT} solver, {Google} {OR-Tools} v9.14,'' 2025, {Accessed: Jun. 13, 2026}. [Online]. Available: \url{https://developers.google.com/optimization/cp/cp_solver/}
\BIBentrySTDinterwordspacing

\bibitem{chan2002production}
W.~T. Chan and H.~Hu, ``Production scheduling for precast plants using a flow shop sequencing model,'' \emph{J. Comput. Civ. Eng.}, vol.~16, no.~3, pp. 165--174, 2002.

\bibitem{wang2023production}
S.~Wang and X.~Zhang, ``Production scheduling of prefabricated components considering delivery methods,'' \emph{Sci. Rep.}, vol.~13, 2023, {Art. no. 15094}.

\bibitem{kim2022rlprecast}
T.~Kim, Y.-W. Kim, D.~Lee, and M.~Kim, ``Reinforcement learning approach to scheduling of precast concrete production,'' \emph{J. Clean. Prod.}, vol. 336, 2022, {Art. no. 130419}.

\bibitem{zu2025rlprecast}
L.~Zu and W.~Liao, ``Reinforcement learning--based multiobjective and multiconstraint production scheduling for precast concrete,'' \emph{J. Constr. Eng. Manage.}, vol. 151, no.~8, 2025, {Art. no. 04025089}.

\bibitem{caumond2008memetic}
A.~Caumond, P.~Lacomme, and N.~Tchernev, ``A memetic algorithm for the job-shop with time-lags,'' \emph{Comput. Oper. Res.}, vol.~35, no.~7, pp. 2331--2356, 2008.

\bibitem{gonzalez2015scatter}
M.~A. Gonz{\'a}lez, A.~Oddi, R.~Rasconi, and R.~Varela, ``Scatter search with path relinking for the job shop with time lags and setup times,'' \emph{Comput. Oper. Res.}, vol.~60, pp. 37--54, 2015.

\bibitem{aallaoui2026graspxvns}
S.~Aallaoui, A.~Azzamouri, L.~Ren, and N.~Tchernev, ``A {GRASPxVNS} approach for the flexible job shop scheduling with generic minimal time lags,'' \emph{Comput. Ind. Eng.}, vol. 216, 2026, {Art. no. 111948}.

\bibitem{zhang2020l2d}
C.~Zhang, W.~Song, Z.~Cao, J.~Zhang, P.~S. Tan, and C.~Xu, ``Learning to dispatch for job shop scheduling via deep reinforcement learning,'' in \emph{Adv. Neural Inf. Process. Syst. ({NeurIPS})}, vol.~33, 2020, pp. 1621--1632.

\bibitem{song2023fjsp}
W.~Song, X.~Chen, Q.~Li, and Z.~Cao, ``Flexible job-shop scheduling via graph neural network and deep reinforcement learning,'' \emph{IEEE Trans. Ind. Informat.}, vol.~19, no.~2, pp. 1600--1610, 2023.

\bibitem{wang2024daniel}
R.~Wang, G.~Wang, J.~Sun, F.~Deng, and J.~Chen, ``Flexible job shop scheduling via dual attention network-based reinforcement learning,'' \emph{IEEE Trans. Neural Netw. Learn. Syst.}, vol.~35, no.~3, pp. 3091--3102, 2024.

\bibitem{lei2022multiaction}
K.~Lei, P.~Guo, W.~Zhao, Y.~Wang, L.~Qian, X.~Meng, and L.~Tang, ``A multi-action deep reinforcement learning framework for flexible job-shop scheduling problem,'' \emph{Expert Syst. Appl.}, vol. 205, 2022, {Art. no. 117796}.

\bibitem{lei2024largescale}
K.~Lei, P.~Guo, Y.~Wang, J.~Zhang, X.~Meng, and L.~Qian, ``Large-scale dynamic scheduling for flexible job-shop with random arrivals of new jobs by hierarchical reinforcement learning,'' \emph{IEEE Trans. Ind. Informat.}, vol.~20, no.~1, pp. 1007--1018, 2024.

\bibitem{kannan2026plan}
D.~D. Kannan, W.~Zhang, J.~Bi, Y.~Du, T.~Wei, J.~Zhang, Z.~Liu, and A.~Trivedi, ``{PLAN}: Parallel liquid-inspired approximation network for efficient representation learning in flexible job shop scheduling,'' 2026, arXiv:2608.03041.

\bibitem{schulman2017ppo}
J.~Schulman, F.~Wolski, P.~Dhariwal, A.~Radford, and O.~Klimov, ``Proximal policy optimization algorithms,'' 2017, arXiv:1707.06347.

\bibitem{liu2024dynamicgat}
C.-L. Liu, C.-J. Tseng, and P.-H. Weng, ``Dynamic job-shop scheduling via graph attention networks and deep reinforcement learning,'' \emph{IEEE Trans. Ind. Informat.}, vol.~20, no.~6, pp. 8662--8672, 2024.

\bibitem{luo2022realtime}
S.~Luo, L.~Zhang, and Y.~Fan, ``Real-time scheduling for dynamic partial-no-wait multiobjective flexible job shop by deep reinforcement learning,'' \emph{IEEE Trans. Autom. Sci. Eng.}, vol.~19, no.~4, pp. 3020--3038, 2022.

\bibitem{smit2025stochastic}
I.~G. Smit, Y.~Wu, P.~Troubil, Y.~Zhang, and W.~P.~M. Nuijten, ``Neural combinatorial optimization for stochastic flexible job shop scheduling problems,'' in \emph{Proc. {AAAI} Conf. Artif. Intell.}, vol.~39, no.~25, 2025, pp. 26\,678--26\,687.

\bibitem{garey1976complexity}
M.~R. Garey, D.~S. Johnson, and R.~Sethi, ``The complexity of flowshop and jobshop scheduling,'' \emph{Math. Oper. Res.}, vol.~1, no.~2, pp. 117--129, 1976.

\bibitem{aci347}
{ACI Committee 347}, ``Guide to formwork for concrete,'' American Concrete Institute, Farmington Hills, MI, USA, ACI 347-04, 2004.

\bibitem{is456}
{Bureau of Indian Standards}, ``Plain and reinforced concrete: Code of practice,'' Bureau of Indian Standards, New Delhi, India, IS 456:2000, 2000.

\bibitem{ss637}
{Singapore Standards Council}, ``Code of practice for waterproofing of reinforced concrete buildings,'' Enterprise Singapore, Singapore, SS 637:2018, 2018.

\bibitem{astmd5957}
{ASTM International}, ``Standard guide for flood testing horizontal waterproofing installations,'' ASTM International, West Conshohocken, PA, USA, ASTM D5957-98(2013), 2013.

\bibitem{nippon9000tds}
{Nippon Paint Malaysia}, ``{Nippon Paint 9000 Gloss Finish},'' Technical data sheet, {Kuala Lumpur, Malaysia}.

\bibitem{jotunjotafloortds}
{Jotun A/S}, ``{Jotafloor Coating},'' Technical data sheet, {Sandefjord, Norway}.

\end{thebibliography}

\clearpage
\includepdf[pages=-]{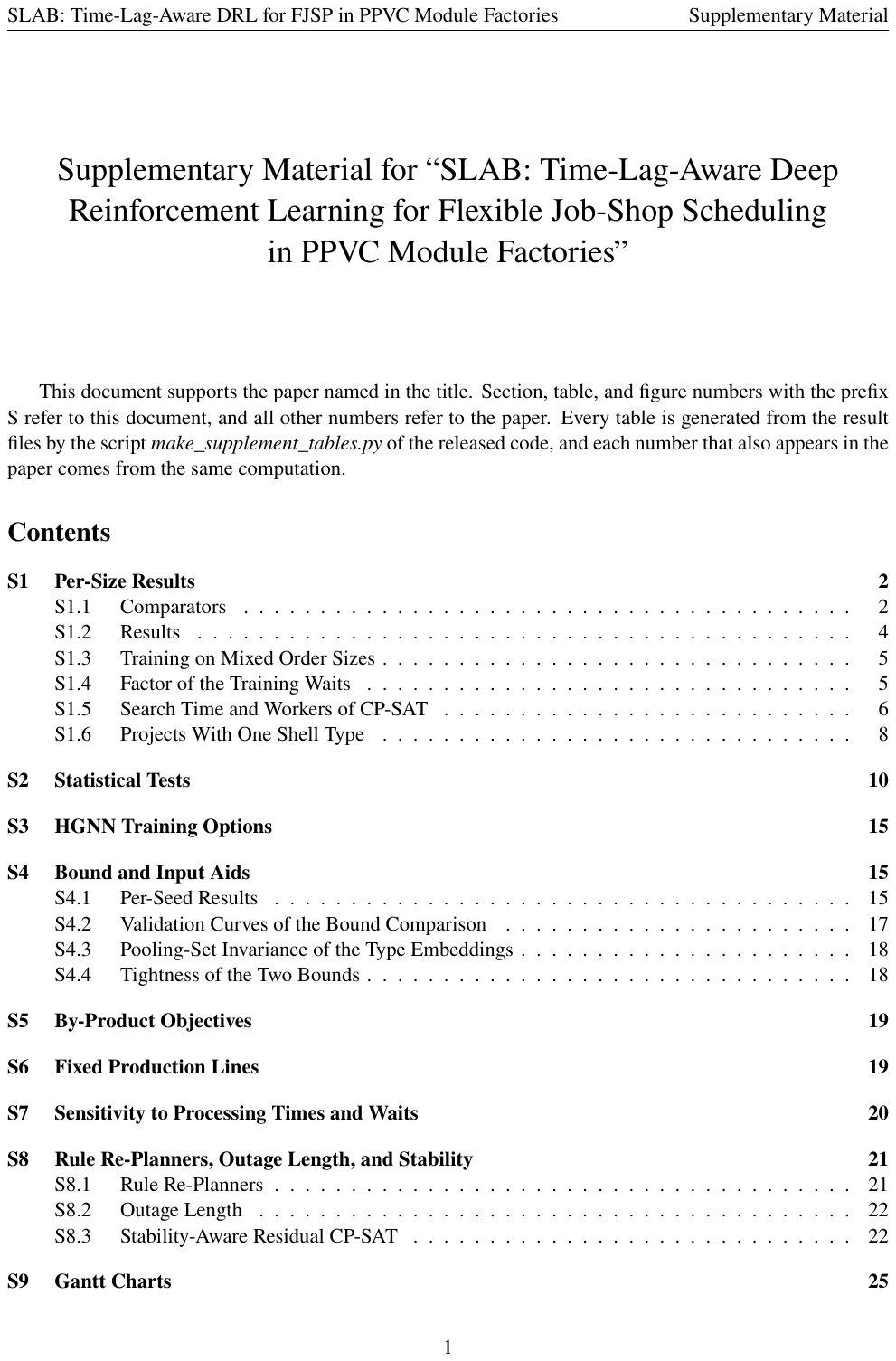}
\end{document}